\documentclass[10pt,twocolumn,letterpaper]{article}

\usepackage{cvpr}              
\definecolor{cvprblue}{rgb}{0.21,0.49,0.74}
\usepackage[pagebackref,breaklinks,colorlinks,allcolors=cvprblue]{hyperref}

\def\paperID{}
\def\confName{CVPR}
\def\confYear{2026}

\title{Attribution-Guided Model Rectification of Unreliable Neural Network Behaviors}

\author{
Peiyu Yang$^{1}$,
Naveed Akhtar$^{1}$,
Jiantong Jiang$^{1}$\thanks{Jiantong Jiang is the corresponding author.},
Ajmal Mian$^{2}$\\
$^{1}$The University of Melbourne,
$^{2}$The University of Western Australia\\
{\tt\small \{peiyu.yang, naveed.akhtar1, jiantong.jiang\}@unimelb.edu.au, ajmal.mian@uwa.edu.au}
}

\usepackage{amsmath}
\usepackage{amsthm}
\usepackage{bbding}
\usepackage{graphicx}
\usepackage{multirow}
\usepackage{subcaption}
\usepackage{wrapfig}

\usepackage{bbding}
\usepackage{enumitem}
\usepackage{booktabs}
\usepackage[ruled, vlined, linesnumbered]{algorithm2e}
\SetKwComment{Comment}{/*}{*/}
\SetKwInOut{Input}{input}
\SetKwInOut{Output}{output}

\SetCommentSty{mycommfont}
\definecolor{dkgreen}{rgb}{0,0.6,0}
\makeatletter
\DeclareRobustCommand\onedot{\futurelet\@let@token\@onedot}
\def\@onedot{\ifx\@let@token.\else.\null\fi\xspace}
\def\cf{\emph{cf}\onedot} 

\usepackage[table]{xcolor}
\usepackage{colortbl}

\definecolor{darkgreen}{RGB}{0, 128, 0}
\definecolor{darkred}{RGB}{220, 0, 0}
\newcommand{\add}[1]{_{\textbf{\textcolor{darkred}{\,+\,#1}}}}

\def\ty{\tilde{y}}
\def\tx{\tilde{x}}
\def\tf{\tilde{f}}

\def\hx{\hat{x}}

\theoremstyle{plain}

\newtheorem{theorem}{Theorem}[section]
\newtheorem{proposition}[theorem]{Proposition}
\newtheorem{lemma}{Lemma}

\theoremstyle{definition}

\theoremstyle{remark}

\begin{document}
\maketitle
\begin{abstract}
The performance of neural network models deteriorates due to their unreliable behavior on non-robust features of corrupted samples. Owing to their opaque nature, rectifying models to address this problem often necessitates arduous data cleaning and model retraining, resulting in huge computational and manual overhead.
In this work, we leverage rank-one model editing to establish an attribution-guided model rectification framework that effectively locates and corrects model unreliable behaviors.
We first distinguish our rectification setting from existing model editing, yielding a formulation that corrects unreliable behavior while preserving model performance and reducing reliance on large budgets of cleansed samples. We further reveal a bottleneck of model rectifying arising from heterogeneous editability across layers. To target the primary source of misbehavior, we introduce an attribution-guided layer localization method that quantifies layer-wise editability and identifies the layer most responsible for unreliabilities. 
Extensive experiments demonstrate the effectiveness of our method in correcting unreliabilities observed for neural Trojans, spurious correlations and feature leakage. Our method shows remarkable performance by achieving its editing objective with as few as a single cleansed sample, which makes it appealing for practice.
\end{abstract}

\section{Introduction}
Neural network models exhibit unreliable behaviors in adapting to inherent or deliberately introduced distributional inconsistencies~\citep{arjovsky2019invariant,lapuschkin2019unmasking,gu2019badnets,liu2025difflow3d}. Such inconsistencies, resulting from, e.g., neural Trojans and spurious features, can misguide a model and alter its behavior from the correct decision-making pathway~\citep{ye2024spurious,gu2019badnets,Shah2021Do}. This compromises model reliability and robustness. Due to the inherent opacity of deep models, primary strategies for correcting such unreliable behavior involve data cleaning and model retraining~\citep{ross2017right,yang2024regulating,anders2022finding,feng2024clipcleaner,Feng2022SSR}. However, these techniques necessitate both labor-intensive manual data scrutiny and substantial computational overheads~\citep{feuer2024select,jiang2024efficient,li2025frequency}. Consequently, efficient techniques for correcting unreliable model behaviors emerge as a critical requirement for enhancing their reliability and sustaining the performance of developed models.

Originally proposed for editing generative rules~\citep{bau2020rewriting,tewel2023key}, rank-one model editing is then applied to revise model prediction rules in discriminative models~\citep{santurkar2021editing,raunak2022rank}.
We revisit this technique for model rectification. Rather than injecting new knowledge or adapting to a new domain, we use rank-one editing to rectify unreliable behaviors, distinguishing from existing domain adaptation settings. Our analysis shows that the standard domain adaptation setting for rank-one editing suffers from two structural limitations, namely out-of-span residuals and increased sample complexity under distribution shift, and that these limitations can degrade the edited model. 
In contrast, our rectification setting enjoys \emph{rectifiability}, which enables correction from a small cleansed set, and \emph{span-aligned control}, which restricts updates to the corrupted span to preserve performance.

\begin{figure*}[t]
\centerline{\includegraphics[width=0.87\textwidth]{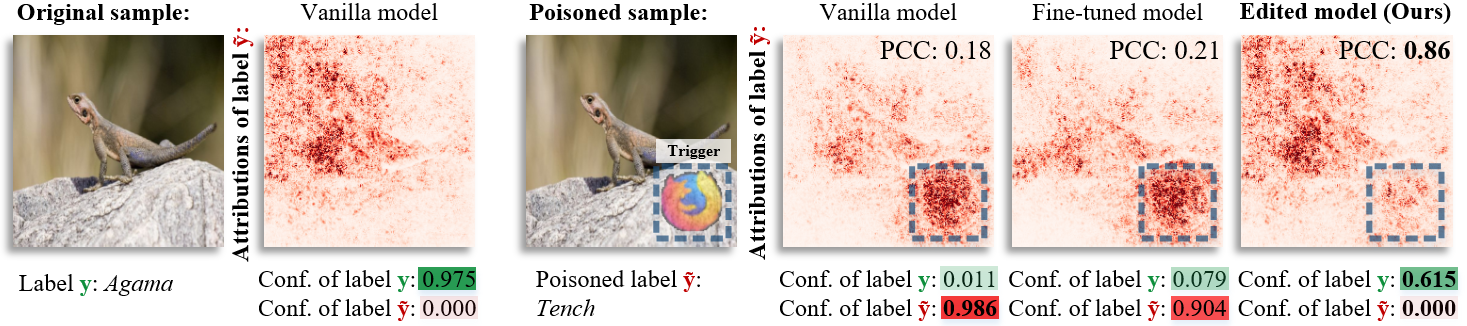}}
\vspace{-1mm}
\caption{Given the original sample labeled as \textit{Agama}, i.e., class \textcolor{darkgreen}{$\textbf{y}$}, the Trojaned model can correctly classify this sample. However, it misclassifies the poisoned sample containing a trigger as \textit{Tench}, i.e., class \textcolor{red}{$\tilde{\textbf{y}}$}. Attribution maps with Pearson Correlation Coefficients (PCCs) and predictive confidence for the vanilla model, fine-tuned model, and model rectified with our approach are provided. Our method restores the correct label by assigning appropriate attributions to the correct object.} 
\vspace{-2mm}
\label{fig_intro}
\end{figure*}

Another central bottleneck of model rectifying lies in deciding \emph{where} to edit. Existing works on model editing usually operate on a fixed last feature layer, motivated by its high-level feature encoding capabilities~\citep{santurkar2021editing,raunak2022rank}. Our analysis of layer editability shows that the effect of a rank-one update varies across layers, and no single layer is universally optimal. Hence, we introduce an attribution-guided estimator to quantify the first-order reduction in unreliability obtained by editing a layer. Concretely, we measure the attribution between a corrupted input $\tilde{x}$ and its cleansed counterpart $x$, and remap these attributions into the rank-one editing direction induced by the key statistics. This procedure yields a scalar score that serves as a proxy to locate the layer that is responsible for the unreliable behavior for rectification. We then integrate this layer localization into a dynamic model rectification framework that chooses a layer adaptively for the rectification of an unreliable model behavior.
Figure~\ref{fig_intro} shows an example showcasing the abilities of our  approach in correcting the model decision for a manipulated sample, as evidenced by the attribution maps, Pearson Correlation Coefficients (PCCs), and confidence scores.

The efficacy of our approach is established through experiments for three model unreliabilities: neural Trojans~\citep{chen2017targeted}, spurious correlations~\citep{ye2024spurious} and feature leakage~\citep{Shah2021Do}. Experimental evaluations highlight our method's performance, offering an excellent trade-off between model performance and the number of utilized cleansed samples. Notably, our method also achieves high  performance with only one cleansed sample. We also extend our assessment to the real-world scenario of skin lesion analysis using the ISIC dataset~\citep{codella2019skin}, thereby illustrating the broader applicability of our approach in practical settings.
The key contributions of this paper are summarized below.
\begin{enumerate}[leftmargin=*,itemsep=2pt,topsep=3pt,parsep=0pt]
\item We cast rank-one model editing as model rectifying for model unreliable behaviors, enjoying properties of rectifiability and span-aligned control.
\item We propose an attribution-based method that provides a surrogate estimate of the layer editability and automatically localizes the suspect layer.
\item We devise a dynamic model rectifying framework incorporating the proposed layer localization method. Efficacy of the approach is verified extensively on different model unreliabilities across diverse datasets.
\end{enumerate}

\section{Related Work}
\textbf{Unreliable Model Behaviors.} \textcolor{black}{Despite their impressive performance, 
neural network models have been found to exhibit numerous unreliable behaviors that lead to incorrect predictions on corrupted samples.} For instance, the existence of spurious correlations, also known as Clever Hans behavior~\citep{pfungst1911clever,zhou2025robustness}, pose a substantial threat to the reliability of these models. 
A range of spuriously correlated features have been identified including object backgrounds~\citep{xiao2020noise}, hair color~\citep{Sagawa2019Distributionally} and colored patches~\citep{gutman2016skin}.
In addition to inherent bias, training data can be intentionally poisoned by mislabeling samples and adding trigger patterns to mislead model predictions \citep{gu2019badnets}. 
More attacks~\citep{chen2017targeted,li2021invisible,turner2019label,yang2024backdoor} are proposed to implant invisible triggers for concealed backdoors. 
Adversarial attacks have demonstrated a significant capacity to alter model predictions~\citep{Goodfellow2015Explaining,Madry2018Towards,Feng2025PROSAC}. However, their practical applicability is often constrained by the necessity of full access to the target model. Consequently, this paper focuses specifically on investigating backdoor attacks and spurious correlations, recognizing their significant impact on undermining the security of deep learning models.

\noindent\textbf{Model Explaining and Diagnosis.}
Various techniques have been proposed to explain and diagnose the vulnerable behaviors of deep models. Attribution methods~\cite{selvaraju2017grad,Sundararajan2017Axiomatic,yang2023local,yang2023recal} assign importance to each input feature to provide explanations for model predictions, which are widely utilized for visually inspecting model  behavior~\citep{lapuschkin2019unmasking,li2021invisible}.
Other efforts from a complementary optimization perspective~\citep{lapuschkin2019unmasking,anders2022finding,ouyang2024learn} are also made to diagnose unreliable behavior in trained models. For instance, SpRAy~\citep{lapuschkin2019unmasking} analyses heatmaps of training samples to identify Clever Hans behaviors. \citet{anders2022finding} proposed A-ClArC and P-ClArC to prevent the propagation of artifact signals. The statistics of internal activations are also widely used in revealing Trojans~\citep{tran2018spectral,hayase2021spectre,qi2022revisiting}.

\section{Preliminary}

Model editing~\citep{bau2020rewriting,meng2022locating} focuses on editing a specific model prediction rule while preserving the learned rules. When examining the $l$-th layer of a model $f$, an input sample $x$ is mapped to a feature map $f_{l}(x)$, which encapsulates semantic concepts in the representation space~\citep{anderson1972simple,kohonen2012associative}. This understanding is extended to characterize a layer as a linear associative memory. Specifically, assuming a location of the input feature $f_{l-1}(x)$ to be a ``key" $k\in \mathbb{R}^n$, the weights $W\in \mathbb{R}^{m\times n}$ within the $l$-th layer map this key $k$ to a ``value" $v\in \mathbb{R}^m$ of output features, achieved via the operation $v=Wk$. Considering a finite set of key-value pairs $K=[k_1, k_2, \dots]$ and $V=[v_1, v_2, \dots]$, we can uniquely retrieve a value from a key if the keys are mutually orthogonal. Beyond the exact equality, weights $W$ can be extended to arbitrary non-orthogonal keys by minimizing the error as
$W = \arg \min_W\sum_{i}\|v_i-Wk_i\|^2$.
Given this characteristic, \citet{bau2020rewriting} edited $W$ to associate a key $k^*$ with a new value $v^*$, effectively rewriting generative model rules.

Recent studies, inspired by the efficacy of the editing technique demonstrated in generative models~\citep{bau2020rewriting,tewel2023key}, apply this paradigm to discriminative models~\citep{santurkar2021editing,raunak2022rank}. \citet{santurkar2021editing} enhanced the domain adaptation capability of classifiers by modifying their prediction rules. For instance, in the case where a ``car" classifier struggles to recognize cars featuring ``wooden wheels", the model's rules are edited to establish an association between the ``wooden wheels" feature and the corresponding activations of ``car", enabling the recognition of cars equipped with wooden wheels. While incorporating a new key-value pair, it is critical to ensure the preservation of previously learned associations. Consequently, this editing process is formulated as a constrained least squares problem that creates a new key-value associative memory, and  preserves the established key-value associations as
%
\begin{equation} \label{eq:rome}
\min_{\Lambda} \|v^* - f_l(k^*; W')\|
\quad \text{s.t.} \enspace  W' = W + \Lambda (C^{-1} k^*)^\top,
\end{equation}
where $C=KK^\top$ is the second moment statistics, $\Lambda \in \mathbb{R}^m$ is the solution. Since $C^{-1}k^*$ and $\Lambda$ are vectors, the update weights $\Lambda (C^{-1} k^*)^\top \in \mathbb{R}^{m\times n}$ is a rank-one matrix. Hence, the editing process in Eq.~\ref{eq:rome} is termed rank-one editing.

\section{Rectifying Model Unreliability}
In this section, we first formalize our \emph{model rectification} setting for three types of unreliable behavior. We then identify two structural properties that distinguish rectification from conventional model editing, and finally show that fixed-layer rank-one editing suffers from an inherent bottleneck.

\subsection{Model Rectification}
\label{sec_model_rect}
We propose leveraging model editing to rectify a model's unreliable behavior. To that end, we consider three suspect behaviors of neural Trojans, spurious correlations, and feature leakage. We treat them under a unified perspective on clean-corrupted pairs $(x,\tilde x)$ with ground-truth label $y$.

\noindent\underline{\textit{Neural Trojaning:}} Given a dataset $\mathcal{D}=\{(x_i,y_i)\}$, neural Trojaning is executed by injecting a portion of poisoned samples $\tx$ with a backdoor trigger and modifying their true label $y$ to an incorrect target $\ty$. After training on these poisoned samples, a Trojaned model $\tf$ is highly likely to misclassify input samples containing the trigger to the label $\ty$.

\noindent\underline{\textit{Feature Spurious Correlation:}}
Feature spurious correlations occur when a classifier $f$ exploits the spurious correlated features in corrupted samples $\tx$ to make predictions. While the model classifies $\tx$ to their correct class  $y$, its reliance on the spurious feature results in a flawed decision pathway, rendering the prediction on such inputs unreliable.

\noindent\underline{\textit{Feature Leakage:}}
Feature leakage refers to the phenomenon where a classifier relies on uninformative artifacts as part of predictive features. In this case, the model makes its predictions in combination with these artifacts in $\tilde{x}$.

All three behaviors instantiate unreliability driven by non-robust cues in $\tx$. To rectify such behavior, we apply a rank-one editing that realigns the mapping from these cues to the clean decision pathway.
When a corrupted sample $\tx$ causes the model to behave unreliably, its cleansed counterpart $x$ can guide the model toward the correct prediction pathway. We designate the input feature derived from the corrupted input $\tx$ as the key $k^*$, and align activations of $k^*$ to the corresponding value $v^*$ mapped from the cleansed sample $x$. We edit the model parameter $W$ to make the feature $k^*$ to yield correct activations $v^*$, thereby correcting the model's unreliable behavior.

Crucially, this \emph{rectification} setting differs from standard model editing for knowledge injection or domain adaptation: here the ``correct'' activation pattern $v^*$ is already realized by the model on clean inputs, and rank-one editing is used only to reroute non-robust cues, rather than to synthesize an entirely new association. We next show that this structural difference yields two favorable properties.

\subsection{Rectifiability and Span-Aligned Control}
We now formalize two properties of our rectification setting: \emph{rectifiability}, which guarantees that the relevant cues lie within the span of the learned key statistics, and \emph{span-aligned control}, which avoids the sample-complexity burden of standard domain adaptation. We first reveal two limitations that arise when rank-one editing is used in a generic domain adaptation scenario, and then show how our rectification setting sidesteps them.

Recall that Eq.~\eqref{eq:rome} updates the weight matrix $W$ using the second-moment statistic $C = KK^{\top}$ of the keys $K=[k_1,\dots,k_d]$. The mapping induced by $C$ is designed to decorrelate the new key $k^*$ from the existing keys $k_i \in K$, thereby preserving previously learned key-value associations during optimization. However, when applying rank-one editing for domain adaptation, the new key $k^*$ is typically \emph{not} part of the statistics used to form $C$, which leads to the following limitation.

\begin{lemma}[Out-of-Span Residual]
\label{lemma1}
For $K = [k_1,...,k_d] \in \mathbb R^{n\times d}$ and $C= KK^{\top}$, when $k^* \not\in \text{span}(K)$, the projection $C^{-1}k^*$ leads to a residual component $C^{-1}r$ outside the span of $K$, measurable by a residual vector $r\in \mathbb{R}^n$.
\end{lemma}

When $k^*$ has a nonzero residual component $r$ outside $\text{span}(K)$, the statistic $C$ does not capture $r$, and the rank-one update based on $C^{-1}k^*$ cannot perfectly decorrelate $k^*$ from the existing keys. This misalignment disrupts established associations and degrades model performance.

A second limitation stems from estimating the new key $k^*$ from an unseen sample $x^*$ drawn from shifted data.

\begin{lemma}[Sample Complexity]
\label{lemma2}
Let $D'\neq D$ and let $\mathcal{X}=\{x_i\}_{i=1}^m$ be samples from $D'$. For a given $x^*\sim D'$ with its target key $k^*$, assume the estimation error satisfies
$\mathbb{E}_{\mathcal{X}}\!\left[\|f(x^*;W_{\mathcal{X}})-k^*\|_2^2\right]\le \delta/m$,
where $\delta$ captures the variability under $D'$ and $W_{\mathcal X}$ is adapted using $\mathcal X$. To ensure the estimation error is below $\varepsilon^2$, necessarily $m \ge \delta/\varepsilon^2$.
\end{lemma}

This lemma formalizes that achieving a small estimation error under distribution shift requires sufficiently many samples from $D'$. In particular, larger variability $\delta$ leads to a larger required sample size $m$ for the same error tolerance.

Our rectification setting process of model rectifying to correct unreliable behaviors involves the susceptible model that integrates both original samples $x$ and their corrupted counterpart $\tx$ into training. Repurposing rank-one editing from domain \emph{adaptation} to \emph{unreliability correction} under this setting removes the out-of-span and sample-complexity obstacles, as captured by the following propositions.

\begin{proposition}[Rectifiability]
\label{prop1}
For a susceptible model trained on both clean and corrupted inputs $(x,\tilde{x})$ for an unreliability case, let $K=[k_1,\dots,k^*]$ collect all resulting keys and $C=KK^{\top}$. Then $k^* \in \text{span}(K)$, and the out-of-span residual $r$ in Lemma~\ref{lemma1} vanishes.
\end{proposition}

With $k^*$ in-span, established key-value associations are preserved, avoiding out-of-span residuals.

\begin{proposition}[Span-Aligned Control]
\label{prop2}
By incorporating $\{x,\tilde{x}\} \in \mathcal{X}$ in training, the model yields $||k^*-f(x^*;W_D)||\rightarrow 0$ as the training error on $\mathcal{X}$ tends to zero, without assuming $|\mathcal{X}|\gg 1$.
\end{proposition}

Proposition~\ref{prop2} shows that paired supervision directly aligns corrupted features with their desired keys on the training set, eliminating the need for a large sample from a shifted distribution.
In summary, our model rectifying framework enjoys \emph{rectifiability} and \emph{span-aligned control}, preserving performance with minimal extra supervision.

\subsection{Layer-Dependent Editability as the Bottleneck}

\setlength{\columnsep}{8pt}
\begin{wrapfigure}{r}{0.55\columnwidth}
\vspace{-4mm}
\centerline{\includegraphics[width=1.0\linewidth]{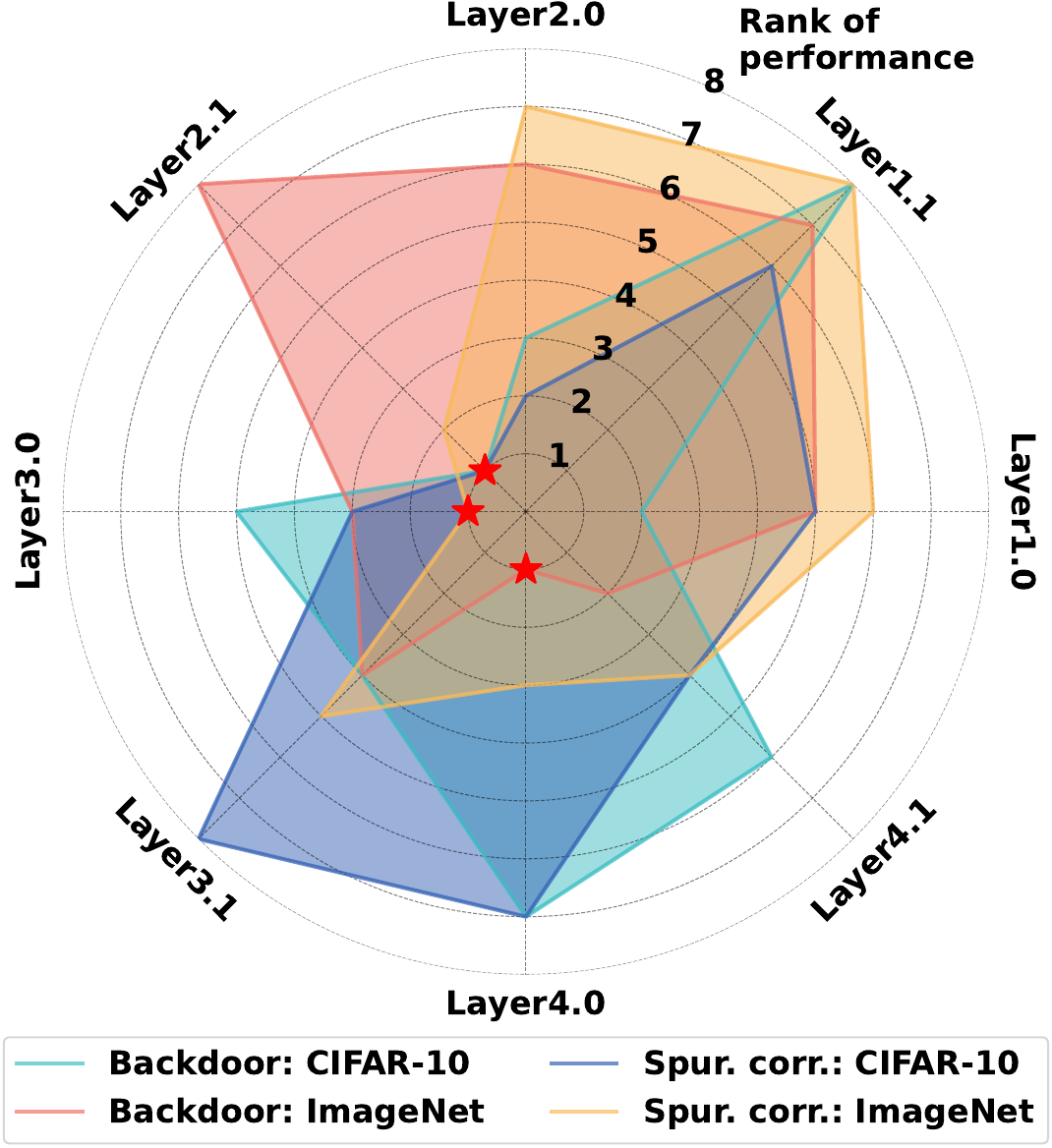}}
	\centering
    \vspace{-3mm}
	\caption{False confidence reduction rank after individually rectifying different layers of ResNet-18. A lower value indicates better results.}
\label{fig_layer_compare}
\vspace{-4mm}
\end{wrapfigure}
Rank-one model editing treats each convolutional layer as a linear associative memory and typically confines the update to a single chosen layer.
Existing methods default to editing the final convolutional layer~\citep{bau2020rewriting,santurkar2021editing}, motivated by its strong encoding of high-level semantic features. However, our analysis reveals that the effectiveness of rectification can vary substantially across layers.
In Fig.~\ref{fig_layer_compare}, we empirically demonstrate that rectifying performed on distinct model layers can yield significantly diverse results when dealing with unreliable model behavior. With the growing popularity of large language model research~\citep{minaee2024large,jiang2025towards,Feng2026NoisyValid,zeng-etal-2025-bridging}, prior works have explored layer localization in the context of language models~\citep{meng2022locating,hase2023does}, but these methods rely on token-level causal analysis, which are absent in vision classifiers and thus render such techniques inapplicable in vision settings.

Formally, let $\Delta(x,\tilde{x};W)$ denote the prediction discrepancy under the model parameter $W=(W_1,\dots,W_L)$. Consider applying a rank-one update to layer $l$ as in Eq.~\eqref{eq:rome}, and let $W^{(l)}(\eta)$ denote the resulting parameters when the rank-one update is scaled by a small step size $\eta$. A first-order Taylor expansion of the discrepancy yields
\begin{equation}
\Delta(x,\tilde{x};\,W^{(l)}(\eta))
\;=\;
\Delta(x,\tilde{x};W) - \eta\,G_l(x,\tilde{x}) + o(\eta),
\end{equation}
where
$G_l(x,\tilde{x}) = - \frac{d}{d\eta}\Delta(x,\tilde{x};W^{(l)}(\eta))$ evaluated at $\eta=0$ is the layer-$l$ first-order decrease coefficient, where $o(\eta)$ collects higher-order terms. Fix any layer $l$. Since the decrease coefficient $G_l(x,\tilde{x})$ depends on $l$, the optimal $l^{\star}=\arg\max_{l'} G_{l'}(x,\tilde{x})$ need not coincide with $l$.
Therefore, any prior choice of a fixed editing layer is not guaranteed to be optimal~\cite{santurkar2021editing}. Consequently, this \emph{layer-dependent editability} constitutes a common bottleneck for model rectification.
In~\S\ref{sec_layer_loc}, we therefore develop a practical procedure to estimate $G_l(x,\tilde{x})$ and to \emph{locate} the layer that offers the largest first-order reduction in the prediction gap, thereby overcoming the bottleneck of fixed-layer editing.


\section{Dynamic Model Rectification}
In this section, we first develop an attribution-guided procedure to yield a scalar editability score of model layers that estimates the first-order decrease coefficient. We then use this score to locate the layer that offers the largest expected reduction in unreliability, and finally integrate this localization into a dynamic model rectification framework.

\subsection{Locating Suspect Layer with Attribution}\label{sec_layer_loc}

In the previous section, we showed the layer-$l$ first-order decrease coefficient $G_l(x,\tilde{x})$ correlating with the editability of a layer $l$. To this end, we leverage integral-based attribution~\citep{Sundararajan2017Axiomatic,Chen2019Robust} to quantify how the model’s prediction changes when moving from a corrupted input $\tx$ to its cleansed counterpart $x$. Integral-based methods, such as Integrated Gradients (IG)~\citep{Sundararajan2017Axiomatic}, compute feature attributions by integrating gradients along a path from a designated reference to the input. In our setting, the corrupted input $\tx$ naturally serves as the reference: interpolating from $\tx$ to $x$ explicitly traces the effect of removing non-robust cues and reinstating the clean decision pathway.

We thus define $\tx$ as the reference and compute the attribution of the prediction change $f(x)-f(\tx)$ across all internal layers of $f$. For the $l$-th layer, we write the attribution $M^l$ from $\tx$ to $x$ as%
\begin{equation} \label{eq:ig}
M_i^l(x, \tx) = (f_l(x_i) - f_l(\tx_i)) \cdot \int_{\alpha=0}^1 \frac{\partial f(\hx)}{\partial f_l(\hx_i)} \bigg|_{\hx=\tx+\alpha(x-\tx)} \mathrm{d}\alpha,
\end{equation}
where $f_l(x_i)$ indicates the $i$-th output feature of the $l$-th layer in $f$, and $\hx$ indicates the interpolated input from the reference input $\tx$ to the input $x$ along a linear path defined by $\alpha$. Attribution is estimated by accumulating the gradient $\partial f(\hx)/ \partial f_l(\hx_i)$ of the interpolated inputs.

The attribution maps computed for different internal layers have  diverse dimensionalities, which  complicates their comparison across the layers. To address this, we leverage the {Completeness} axiom \citep{Sundararajan2017Axiomatic}  to enable   the sought comparability of the attributions. 
The axiom asserts that the sum of attributions equals the model prediction change from the reference to the input, i.e., $\sum_i M_i = f(x) - f(\tx)$. We extend this axiom to the internal layers of the model through the following lemma.

\begin{lemma}
\label{lemma3}
For the $l$-th internal layer $f_l$ of model $f$,  $\sum_i M^l_i = f(x) - f(\tx)$, where $l\in \{1, \dots, L\}$.
\end{lemma}

The proof of Lemma~\ref{lemma3} is provided in App.~\ref{appd:proof}. This lemma shows that the cumulative attribution of each layer accounts for the same global prediction change $f(x)-f(\tx)$, which allows us to compare layers on equal footing.

Next, we connect these attributions to rank-one editing and the coefficient $G_l(x,\tilde{x})$. The rank-one update at a chosen layer operates in the direction $C^{-1}k^*$ determined by the key statistics $C$ and the new key $k^*$ associated with $\tilde{x}$. Intuitively, the first-order change in $\Delta(x,\tilde{x};W)$ induced by such an update is governed by the directional derivative of $\Delta$ along this editing direction. By the chain rule, this directional derivative can be expressed in terms of the gradient of the output with respect to layer-$l$ features, which is precisely what the integrated gradients in Eq.~\eqref{eq:ig} summarize along the path from $\tx$ to $x$.

Motivated by this observation, we remap the layer-wise attribution into the editing direction via
$
M^{*,l}(x,\tx)
=
M^l(x,\tx)\,(C^{-1}k^*)^{\top}.
$
The map $M^{*,l}$ can be viewed as an attribution-weighted projection onto the rectification direction at layer $l$. We then define the scalar score as
$\hat{G}_l(x,\tilde{x})=\big\|M^{*,l}(x,\tilde{x})\big\|_F,$
and use it as a \emph{proxy} for the magnitude of the first-order decrease coefficient.
Under the local linear approximation, $\widehat{G}_l$ correlates with $|G_l|$, indicating that layers with larger $\widehat{G}_l$ yield larger first-order reductions.

In practice, we therefore select the suspect layer as $
l^\star(x,\tilde{x})
=
\arg\max_l \widehat{G}_l(x,\tilde{x}).
$
This choice is consistent with the objective of maximizing the first-order reduction in the prediction gap, thereby directly instantiating the layer-selection principle.

Figure~\ref{fig_workflow} illustrates the pipeline of the proposed layer localization approach in Step~1\&~2. Given a corrupted sample $\tx$ and its cleansed sample $x$, the model yields predictions $f(x)$ and $f(\tx)$ via two distinct decision pathways in Step~1.
The prediction change is then attributed to the features derived from different internal layers, quantified by attributions $M^l(x, \tx)$. Calculated attributions are further transformed to emphasize editable parameters by mapping them into the space $C^{-1}k^*$. In Step~2, the editable information of attributions across layers is compared to identify the suspect layer primarily responsible for model unreliabilities.

\begin{figure}[t]
\begin{center}
\centerline{\includegraphics[width=0.51\textwidth]{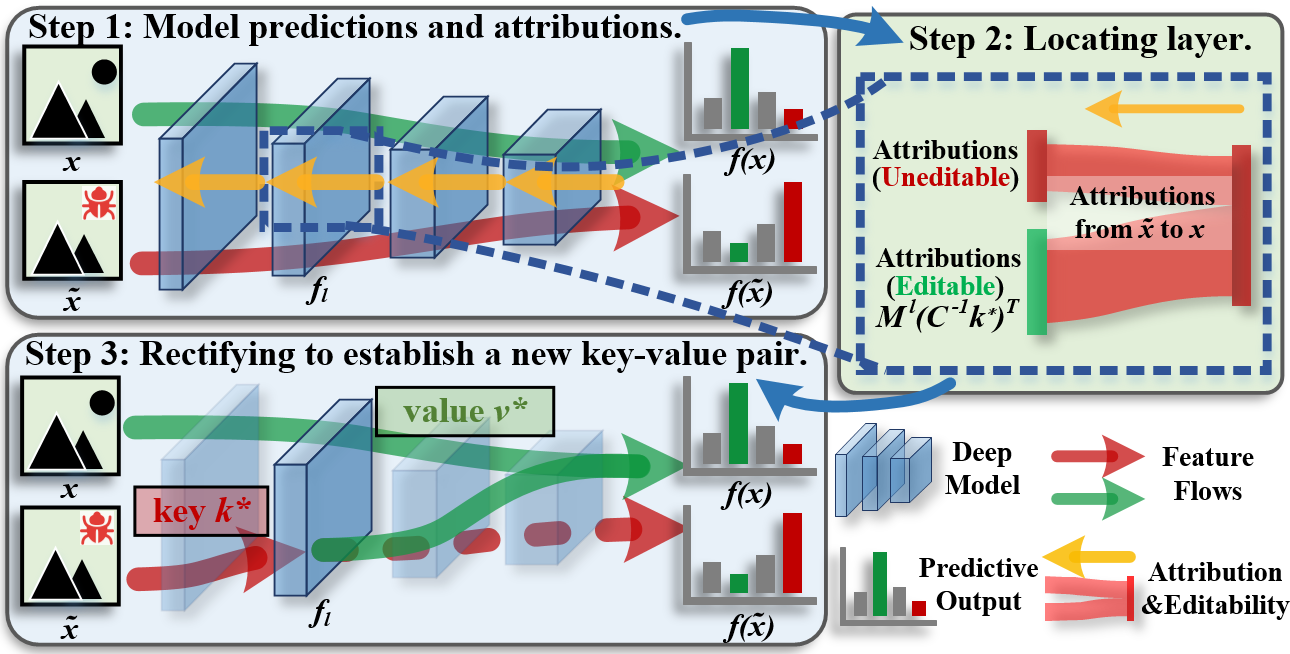}}
\vspace{-1mm}
\caption{Model rectifying workflow.
Step 1: Given a pair of clean and corrupted samples, map their prediction attributions for all internal layers. Step 2: Transformed attributions to emphasize editable parameters and locate the suspect layer. Step 3: Apply Rank-one model editing to establish a new key-value association in the located layer for behavior correction.}
\vspace{-6mm}
\label{fig_workflow}
\end{center}
\end{figure}

\subsection{Dynamic Model Rectification Framework}

It is possible to establish an effective model rectifying technique by modifying the suspect layer identified in the previous section. We illustrate this in Step 3 in Fig.~\ref{fig_workflow}, where by directly applying rank-one model editing to the suspect layer $f_l$, we remap the key $k^*$ from the corrupted sample $\tx$ to the value $v^*$ derived from the cleansed sample $x$. Though effective, this would be a form of \textit{static} rectifying, which ignores the potential model shift during the rectifying process itself. Recognizing the problem, we propose dynamic model rectifying that incorporates our layer localization technique to dynamically identify the suspect layers, and improve them. Our technique facilitates automatic adaptation of the model layers for behavior correction.

Algorithm~\ref{alg_framework} presents the proposed dynamic model rectifying framework. Given a model $f$, the objective is to correct the model's behavior on a corrupted sample $\tx$ by aligning it with the decision pathway of the cleansed sample $x$. Assuming  prediction gap $\delta^*=f(x)-f(\tx)$, we aim to minimize $\delta^*$ to achieve the target gap $\delta$ within an overall budget of $\epsilon$. Specifically, while the current gap $\delta^*$ exceeds the targeted gap $\delta$, and the current performance degradation $\epsilon^*$ remains within the overall budget $\epsilon$ (Line 2), the algorithm identifies the $l$-th layer  responsible for the unreliability on $\tx$ by comparing the editable components of attributions $M^*$ (Line 3). Subsequently, rank-one  editing is applied to establish a new key-value association in the identified layer (Line 4). Following a predefined number of rectifying epochs $T$, we update the current budget $\epsilon^*$ and gap $\delta^*$ based on the evaluation results of the edited model $f^*$ (Lines 5-6). If the overall performance degradation in $f^*$ remains within the permissible overall budget $\epsilon$, the edited model $f^*$ is preserved, and the rectifying process continues (Lines 7-8). Otherwise, $f$ will be returned (Line 9).

\vspace{-1mm}
\begin{algorithm}
    \small
    \DontPrintSemicolon
    \caption{Dynamic Model Rectifying}
    \label{alg_framework}
    \Input{model $f$, overall budget $\epsilon$, targeted gap $\delta$, corrupted sample $\tx$, cleansed sample $x$, rank-one model editing $\Omega$, evaluation metric $\zeta$}
    \textbf{initialize}: $\epsilon^* \leftarrow 0$, $\delta^* \leftarrow f(x)-f(\tx)$. \\
    \While (){$\delta^*>\delta$ and $\epsilon^* \leq \epsilon$}{
        \tcp{locate a layer \cf \S~\ref{sec_layer_loc}}
        $l \leftarrow \arg \max_l\ ||M^{*,l}||_F$\\
        \tcp{model rectifying \cf \S~\ref{sec_model_rect}}
        $f^* \leftarrow \Omega(f_l, T, \tx,x)$\\
        $\epsilon^* \leftarrow \underset{(x,y)\sim D}{\mathbb{E}}\zeta(f(x),y) - \zeta(f^*(x),y)$ \\
        $\delta^* \leftarrow f^*(x) - f^*(\tx) $\\
        \If(){$\epsilon^* \leq \epsilon$}{
            $f \leftarrow f^*$
        }
    }
    \Return{$f$}
\end{algorithm}
\vspace{-2mm}

\section{Experiments}
We conduct extensive empirical validation across diverse datasets to assess the efficacy of our proposed methods. Detailed experimental setups are available in App.~\ref{appd:setup}.

\subsection{Efficacy Against Neural Trojans}

\begin{table}[tb]
\centering
\footnotesize
\vspace{-2mm}
\caption{Rectifying backdoor vulnerability. Overall accuracy (OA) and attack success rate (ASR) are reported for varying number ($n$) of samples. Ours are
\colorbox{dkgreen!20}{highlighted} and the best metrics are in \textbf{bold} (with Trojaned model in \textcolor{gray}{gray} for reference).}
\vspace{-1mm}
\centering
\label{tab_exp_backdoor}
\begin{tabular}{lcccc}
\toprule[1.5pt]
\multirow{2}{*}{Methods} & 
\multicolumn{2}{c}{CIFAR-10} &
\multicolumn{2}{c}{ImageNet} \\ \cmidrule(r){2-3} \cmidrule(r){4-5} & 
OA  $\uparrow$ & ASR $\downarrow$ & OA $\uparrow$ & ASR $\downarrow$ \\
\hline
Trojaned model & \textcolor{gray}{93.67} & \textcolor{gray}{99.94} & \textcolor{gray}{69.05} & \textcolor{gray}{87.24}\\
\hline
Fine-tuned model ($n$=1) & 90.83 & 73.07 & 65.95 & 79.91\\
Fine-tuned model ($n$=10) & 91.57 & 30.14 & 68.66 & 33.73\\
Fine-tuned model ($n$=20) & 91.58 & 13.22 & 68.42 & 21.86\\
\hline
Patched model ($n=20$) & 89.70 & 12.19 & 65.59 & 13.81\\
P-ClArC ($n$=20) & 89.97 & 6.21 & 65.42 & 8.09\\
A-ClArC ($n$=20) & 92.53 & 6.32 & 67.17 & 8.73\\
\hline
Stat. rectifying ($n$=1) & \cellcolor{dkgreen!20} 92.93 & \cellcolor{dkgreen!20} 2.57 & \cellcolor{dkgreen!20} 67.87 & \cellcolor{dkgreen!20} 3.01\\
Dyn. rectifying ($n$=1) & \cellcolor{dkgreen!20} \textbf{93.65} & \cellcolor{dkgreen!20} 1.34 & \cellcolor{dkgreen!20} 66.77 & \cellcolor{dkgreen!20} 1.61\\
Dyn. rectifying ($n$=20) & \cellcolor{dkgreen!20} 93.61 & \cellcolor{dkgreen!20} \textbf{0.26} & \cellcolor{dkgreen!20} \textbf{68.84} & \cellcolor{dkgreen!20} \textbf{0.12}\\
\bottomrule[1.5pt]
\end{tabular}
\end{table}
%
To evaluate the efficacy of our approach, we conduct experiments on Trojaned models using CIFAR-10~\citep{krizhevsky2009learning} and ImageNet~\citep{russakovsky2015imagenet} datasets.
Trojaned models are established by training on a poisoned training set - see App.~\ref{appd:setup} for further details.  We use overall accuracy (OA) and attack success rate (ASR)~\citep{chen2017targeted} as the metrics.

\noindent\textbf{Overall Evaluation.} Table~\ref{tab_exp_backdoor} summarizes extensive results on  different models, including fine-tuned models, patched models~\citep{wang2019neural}, models learned by projective and augmentative class artifact compensation methods (P-ClArC and A-ClArC)~\citep{anders2022finding}.
P-ClArC and A-ClArC mitigate unreliability by adding artifact modules and using $n$ cleansed samples from the original training set. P-ClArC sharply reduces ASR at the cost of substantial accuracy loss, whereas A-ClArC recovers clean accuracy but slightly increases ASR. Pruning-based neuron patching similarly degrades overall performance.
In contrast, our method significantly reduces ASR with minimal cleansed input samples, while retaining high overall accuracy. In the table, we also include the {static} variant of our approach, for which we only rectify the final layer. Notably, the models rectified dynamically consistently outperform those rectified at the final layer, underscoring the effectiveness of our dynamic rectifying method.

\begin{figure}[tb]
     \centering
     \begin{subfigure}[b]{0.4\textwidth}
         \centering
         \includegraphics[width=\textwidth]{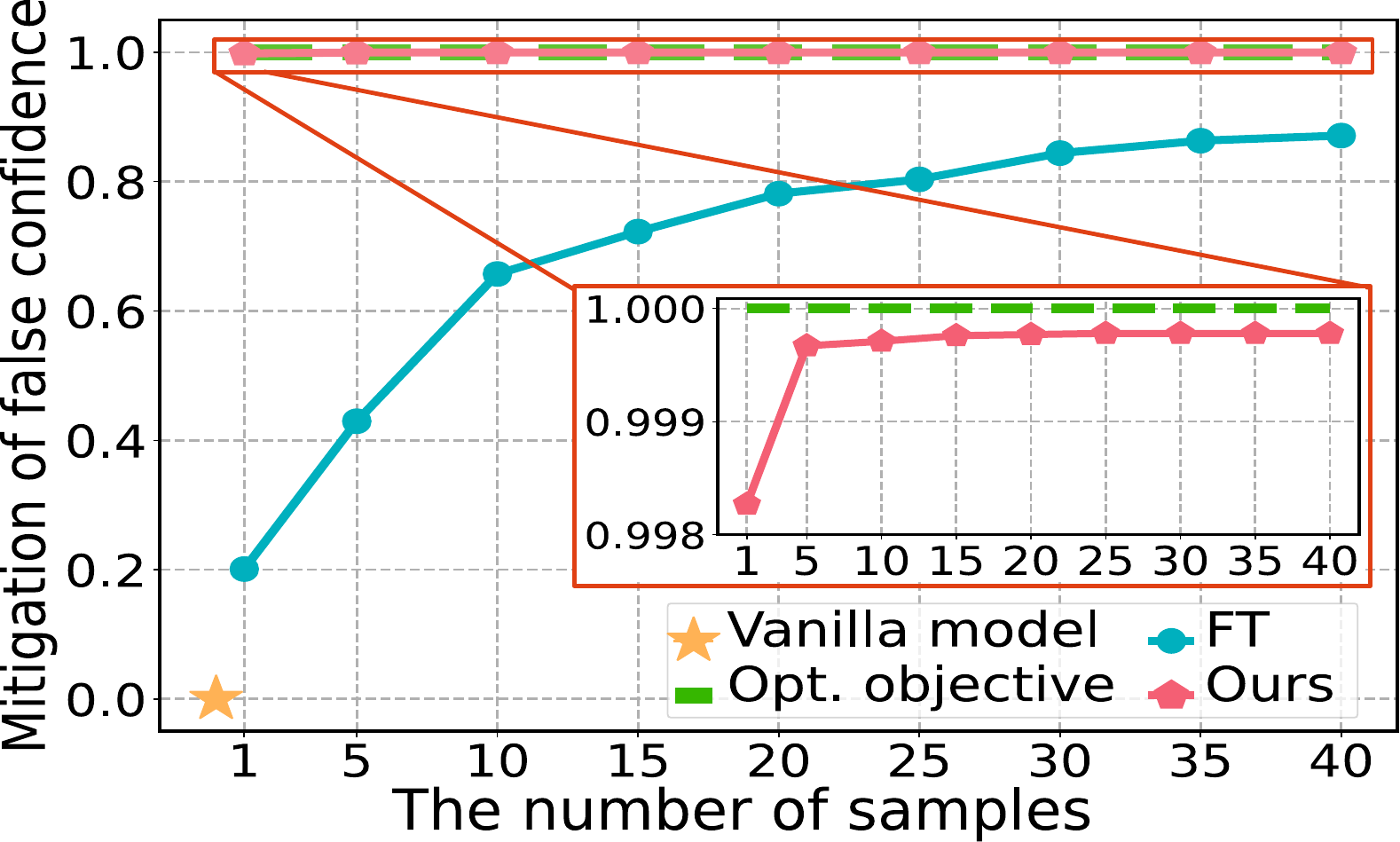}
         \caption{Performance changes with sample size.}
     \end{subfigure}
     \bigskip
     \begin{subfigure}[b]{0.4\textwidth}
         \centering
         \includegraphics[width=\textwidth]{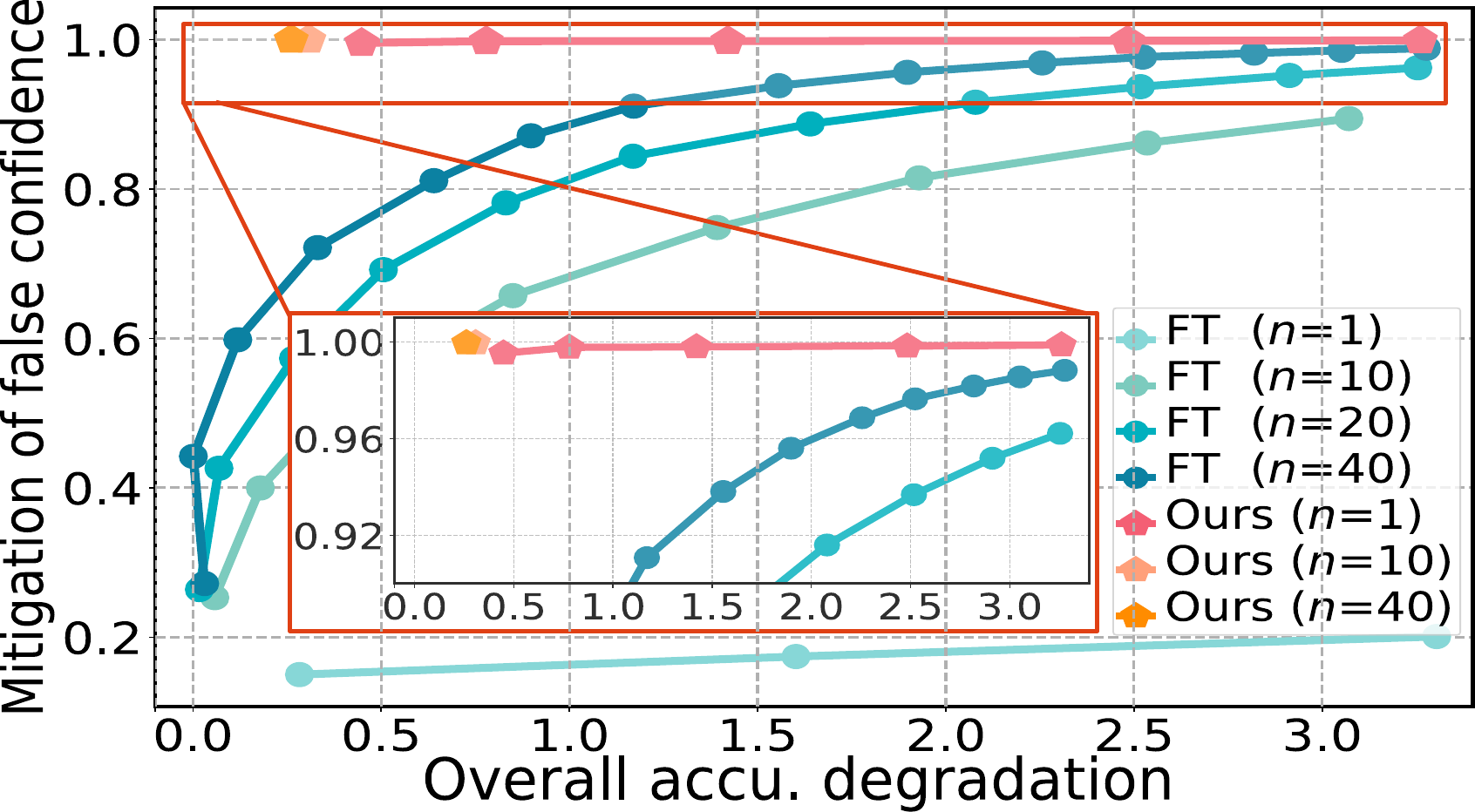}
         \caption{Performance changes with overall accuracy.}
     \end{subfigure}
     \vspace{-3mm}
\caption{Comparison of model performance between fine-tuned models (FT) and rectified models (Ours). (a) The mitigation of false confidence changes with the number of used samples. (b) The mitigation of false confidence changes with the overall accuracy degradation (\%) during model rectifying and fine-tuning. Results are computed for ResNet-18 on CIFAR-10 dataset. 
}
\label{fig_trade_off}
\end{figure}


\noindent\textbf{Trade-off Evaluation.} In Fig.~\ref{fig_trade_off}, we demonstrate the mitigation of false predictive confidence of class $\ty$ by examining how it changes with variations in the number of utilized cleansed samples ($n$) and the overall accuracy degradation during optimization. Remarkably, our methods exhibit outstanding performance even with a single cleansed sample, while resulting in only marginal overall accuracy degradation. In comparison to the fine-tuned (FT) models, our methods showcase an exceptional balance between mitigating false confidence, preserving overall accuracy, and the requirement of cleansed samples.

\noindent\textbf{Generalization Evaluation.} We further evaluate the generalization of our approach for addressing neural Trojans  involving triggers with varying visibilities and spatial locations. First, we train a Trojaned model using poisoned samples with the trigger at visibility levels of $0.3$, $0.5$, $0.7$, and $1.0$. To evaluate how well our method generalizes across different trigger visibilities, we patch and rectify the model using a single corrupted sample. As evidenced in Tab.~\ref{tab_tigger_vis}, our method effectively mitigates triggers of various visibility levels when using the fixed visibility trigger, demonstrating superior performance compared to the patched model.   Next, we evaluate our method's effectiveness in handling triggers placed at different spatial locations. We train a Trojaned model with triggers located at \textcolor{black}{top-left, top-right, center, bottom-left, and bottom-right positions}. We then patch and rectify the model with a single corrupted sample. Table~\ref{tab_tigger_loc} demonstrates that our method successfully handles neural Trojans with triggers located at different positions.


\begin{table}[tb]
\centering
\footnotesize
\caption{Generalization comparison for trigger in different visibility levels (0.3-1.0). A single corrupted sample is used for model patching and rectifying. ASR at visibility $\varphi$ is denoted as $\Gamma_{\varphi}$.}
\vspace{-1mm}
\centering
\label{tab_tigger_vis}
    \begin{tabular}{lccccc}
        \toprule[1.5pt]
        Methods & OA $\uparrow$ & $\mathbf{\Gamma}_{\textbf{0.5}}$  $\downarrow$ & $\Gamma_{0.3}$  $\downarrow$ & $\Gamma_{0.7}$ $\downarrow$ & $\Gamma_{1.0}$ $\downarrow$\\
        \hline
        Benign model & 92.85 & 95.29 & 95.15 & 97.81 & 99.21\\
        Patched model & 89.61 & 26.86 & 30.84 & 32.42 & 37.19 \\
        Dyn. rectifying & \cellcolor{dkgreen!20} 91.21 & \cellcolor{dkgreen!20} 5.17 & \cellcolor{dkgreen!20} 6.84 & \cellcolor{dkgreen!20} 7.65 & \cellcolor{dkgreen!20} 7.91 \\
        \bottomrule[1.5pt]
    \end{tabular}
\vspace{-1mm}
\end{table}

\begin{table}[tb]
\centering
\footnotesize
\caption{Generalization comparison for trigger located at top-left (TL), Center (C) and bottom-left (BL). One corrupted sample with trigger at bottom-right (\textbf{BR}) used for model patching and rectifying. ASR at location $\eta$ is denoted as $\Gamma_{\eta}$.}
\vspace{-1mm}
\centering
\label{tab_tigger_loc}
    \begin{tabular}{lccccc}
        \toprule[1.5pt]
        Methods & OA $\uparrow$ &
        $\mathbf{\Gamma}_{\text{\textbf{BR}}}$ $\downarrow$ &
        $\Gamma_{\text{TL}}$
        $\downarrow$ & $\Gamma_{\text{C}}$ $\downarrow$ &  $\Gamma_{\text{BL}}$ $\downarrow$\\
        \hline
        Benign model & 91.23 & 99.74 & 99.57 & 99.76 & 99.90\\
        Patched model & 89.22 & 29.31 & 34.42 & 34.58 & 34.88\\
        Dyn. rectifying & \cellcolor{dkgreen!20} 90.85 & \cellcolor{dkgreen!20} 6.36 & \cellcolor{dkgreen!20} 9.24 & \cellcolor{dkgreen!20} 9.47 & \cellcolor{dkgreen!20} 8.95\\
        \bottomrule[1.5pt]
    \end{tabular}
\vspace{-1mm}
\end{table}

\subsection{Efficacy in Mitigating Spurious Correlation}
%
In this part, we evaluate the efficacy of our method for mitigating feature spurious correlations. We induce spurious correlations in the model by utilizing class-irrelevant patterns as spurious  features. Specifically, we pollute a proportion of samples of class $y$ by attaching  patterns to create spurious samples $\tx$. After training on the dataset including these samples, the model tends to rely on  spurious features to predict the correct label. In our evaluation, we assess the model performance on two distinct sets of class $y$; namely, the clean set and the spurious set. The latter  encompasses samples containing spurious features. Reliable models are expected to yield consistent accuracy across both the spurious and clean sets, as well as on the overall testing set.

Table~\ref{tab_exp_spurious} summarises the results for addressing the spurious correlation problem. The table shows that the benign model heavily relies on spurious features for predictions, resulting in higher accuracy on spurious set as compared to the clean set.
Fine-tuned models only marginally mitigate spurious correlations while enlarging the performance gap between clean and spurious sets. A-ClArC reduces spurious correlations but degrades performance on both sets. P-ClArC further narrows the disparity between spurious and clean samples, yet drives clean and overall accuracy to unacceptable levels.
In contrast, our methods demonstrate notable effectiveness in mitigating spurious correlations with a limited cleansed set, yielding model accuracy on spurious set that aligns closely with that on clean set. Moreover, dynamic rectified models exhibit heightened efficacy in mitigating spurious correlations. 

\begin{table}[tb]
\footnotesize
\caption{Performance comparison for mitigating spurious correlation. Accuracy is reported for the overall testing, clean, and spurious sets. \textcolor{black}{The erroneously increased accuracy on the spurious set, compared to the accuracy on the clean set, is indicated in \textcolor{darkred}{red}.}}
\vspace{-1mm}
\centering
\label{tab_exp_spurious}
\setlength{\tabcolsep}{0.5mm}{
\scalebox{0.85}{
\begin{tabular}{lcccccc}
\toprule[1.5pt]
\multirow{2}{*}{Methods} & 
\multicolumn{3}{c}{CIFAR-10} &
\multicolumn{3}{c}{ImageNet} \\ \cmidrule(r){2-4} \cmidrule(r){5-7} & 
Overall $\uparrow$ & Clean $\uparrow$ & Spurious & Overall $\uparrow$ & Clean $\uparrow$ & Spurious \\
\hline
Benign model & 94.00 & 94.42 & 100.00$\add{5.58}$ & 69.04 & 81.25 & 91.66$\add{10.41}$\\
\hline
FT model (n=10) & 93.32 & 88.22 & 99.66$\add{11.40}$ & 68.01 & 64.58 & 74.99$\add{10.41}$\\
FT model (n=20) & 93.47 & 88.97 & 99.62$\add{10.65}$ & 68.18 & 64.58 & 74.99$\add{10.41}$\\
\hline
P-ClArC (n=20) & 88.29 & 16.89 & 17.12$\add{0.23}$ & 66.84 & 8.32 & 10.91$\add{2.59}$\\
A-ClArC (n=20) & 92.41 & 76.77 & 79.34$\add{2.57}$ & 67.01 & 75.66 & 82.25$\add{6.59}$\\
\hline
Stat.  rectifying (n=1) & \cellcolor{dkgreen!20} 93.19 & \cellcolor{dkgreen!20} 96.65 & \cellcolor{dkgreen!20} 98.88$\add{2.23}$ & \cellcolor{dkgreen!20} 67.64 & \cellcolor{dkgreen!20} 81.25 & \cellcolor{dkgreen!20} 87.50$\add{6.25}$\\
Dyn.  rectifying (n=1) & \cellcolor{dkgreen!20} 92.93 & \cellcolor{dkgreen!20} 94.29 & \cellcolor{dkgreen!20} 96.15$\add{1.86}$ & \cellcolor{dkgreen!20} 67.50 & \cellcolor{dkgreen!20} 81.66 & \cellcolor{dkgreen!20} 85.83$\add{4.17}$\\
Dyn.  rectifying (n=20) & \cellcolor{dkgreen!20} 93.99 & \cellcolor{dkgreen!20} 94.30 & \cellcolor{dkgreen!20} 94.42$\add{0.12}$ & \cellcolor{dkgreen!20} 68.94 & \cellcolor{dkgreen!20} 81.25 & \cellcolor{dkgreen!20} 83.33$\add{2.08}$\\
\bottomrule[1.5pt]
\end{tabular}}
}
\vspace{-1mm}
\end{table}

\subsection{Efficacy in Mitigating Feature Leakage}
%
\begin{figure}[h!]
    \centering
\includegraphics[width=.48\textwidth]{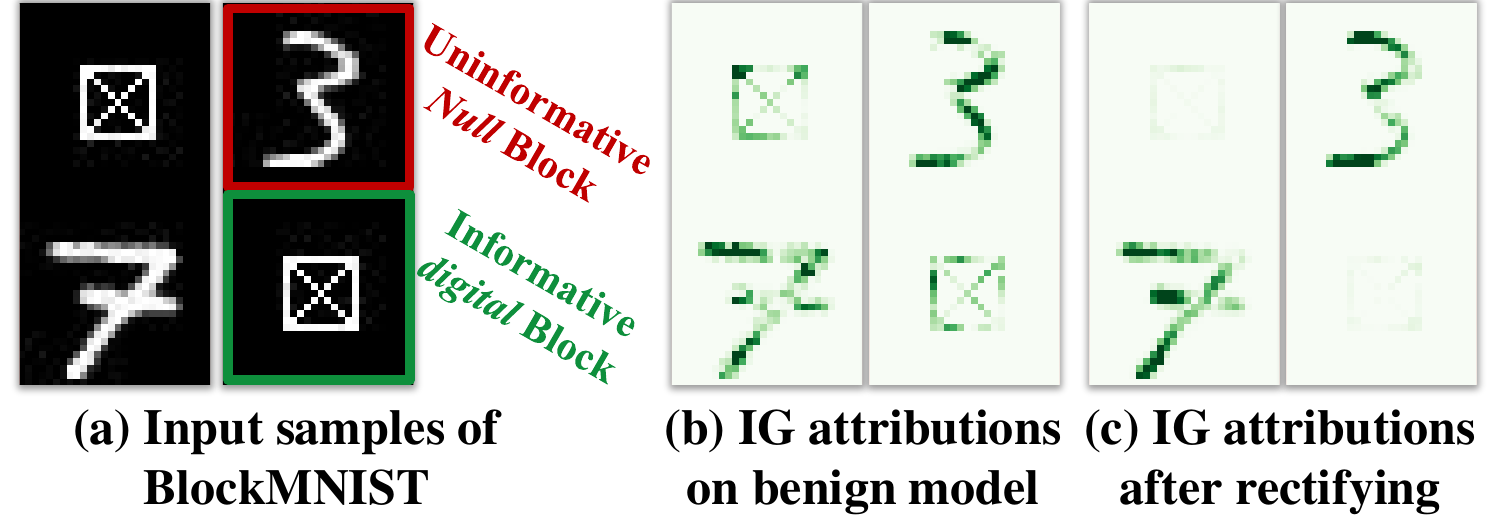}
	\caption{BlockMNIST and feature leakage. \textbf{(a)} Null block randomly appended at the top or bottom of MNIST samples. \textbf{(b\&c)} Integrated gradients estimated on benign and our rectified models.}
	\label{fig:BMnist}
\end{figure}
Prior work has shown that deep networks treat irrelevant patterns as predictive features~\cite{Shah2021Do}. To study this effect, Shah \etal introduce BlockMNIST~\cite{Shah2021Do}, a variant of MNIST~\cite{Lecun1998Gradient} where each digit image is augmented with an additional ``null'' patch placed either above or below the digit, see Fig.~\ref{fig:BMnist}(a). However, the benign model incorrectly incorporates these artifacts into its decision-making process, as illustrated in Fig.~\ref{fig:BMnist}(b). This indicates that informative features leak into an uninformative region, a phenomenon referred to as \emph{feature leakage}. A robust model should not rely on such non-informative artifacts for prediction. We therefore assess the efficacy of our rectification method in correcting this leakage problem.

Table~\ref{tab_bmnist} reports the performance comparison between our method, model fine-tuning, and IG-SUM robust regularization~\cite{Chen2019Robust}. Here, we follow~\cite{yang2024regulating} to quantify the extent of leaked features. The results show that the benign model attains high accuracy but exhibits substantial feature leakage, while IG-SUM slightly reduces leakage at the cost of a notable drop in accuracy. Fine-tuning with 20 cleansed samples further lowers leakage but still underperforms our rectified models. In contrast, both static and dynamic rectification with only a single cleansed sample achieve comparable or higher accuracy and markedly lower leakage, with the dynamic variant delivering the best overall trade-off between robustness and performance. Fig.~\ref{fig:BMnist}(c) also shows that our rectification method significantly suppresses the leakage problem.

\begin{table}
    \caption{Accuracy (\%) and leakage comparison on BlockMNIST.}
    \vspace{-1mm}
    \footnotesize
    \centering
    \label{tab_bmnist}
    \begin{tabular}{lcc}
    \toprule[1.5pt]
    Methods & Accuracy $\uparrow$ & Feature Leakage $\downarrow$ \\
    \hline
    Benign model & 99.17 & 3.597 \\
    Benign model + IG Reg. & 94.14 & 3.417 \\
    FT model (n=20) & 98.67 & 2.929 \\
    \hline
    Stat. rectifying (n=1) & \cellcolor{dkgreen!20} 98.97 & \cellcolor{dkgreen!20} 2.655 \\
    Dyn. rectifying (n=1) & \cellcolor{dkgreen!20} 99.03 & \cellcolor{dkgreen!20} 2.417 \\
    \bottomrule[1.5pt]
    \end{tabular}
\end{table}

\subsection{Evaluation on Skin Lesion Analysis}

\begin{figure}[h!]
    \centering
\includegraphics[width=.45\textwidth]{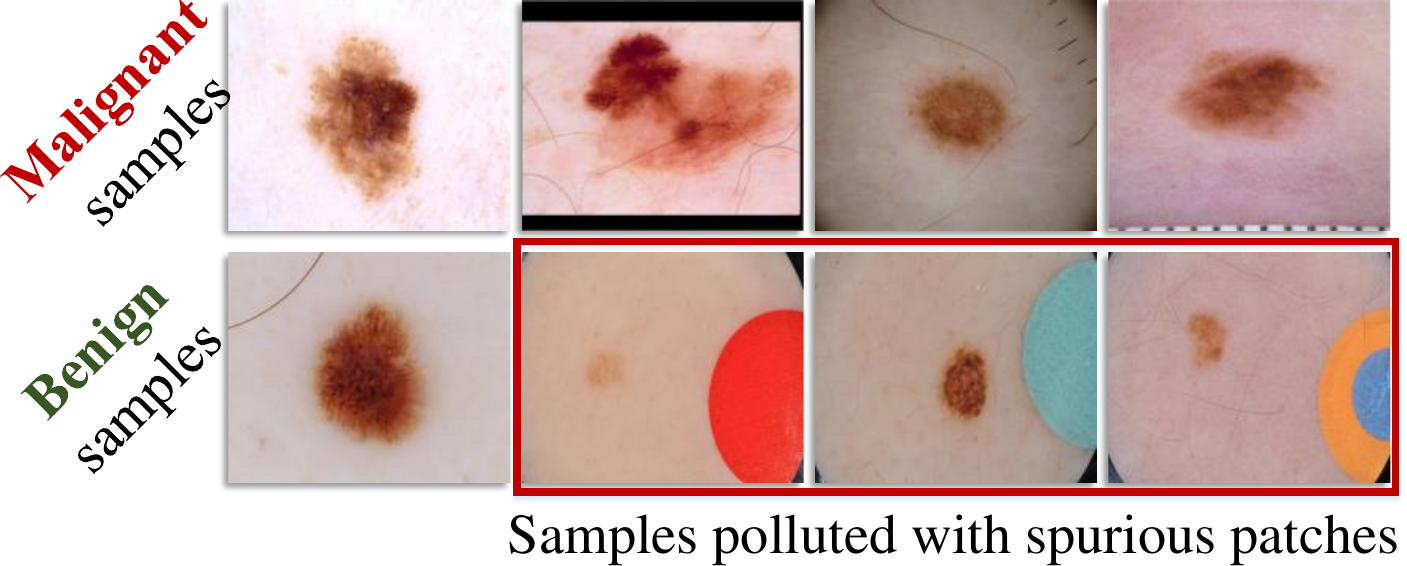}
	\caption{ISIC samples and the inherent spurious patches. Samples containing malignant and benign lesions are presented, where benign samples are partly polluted with spurious colored patches.}
	\label{fig:isic}
\end{figure}
To further assess the broader utility of our method, we applied it to a real-world problem involving skin lesion analysis on the ISIC dataset~\citep{codella2019skin}. Specifically, we conduct a binary classification of the ISIC data to distinguish between \textit{benign} and \textit{malignant} skin lesions, adhering to the setting of \citet{rieger2020interpretations}. In this case,  unreliability in the model arises from the presence of colored patches within the benign samples, which introduce spurious correlations learned by the models. Figure~\ref{fig:isic} shows representative samples from the ISIC dataset, illustrating instances of polluted samples with spurious colored patches. Thus, we employ a manual approach to remove spurious features by replacing the areas affected by colored patches on the skin with cleaned skin from another region.

%
\begin{table}
    \caption{Performance comparison for mitigating spurious correlation on ISIC. For our edited models, we use n=10. The erroneously increased accuracy on the spurious set is indicated in \textcolor{darkred}{red}.}
    \vspace{-1mm}
    \footnotesize
    \centering
    \label{tab_exp_isic}
    \begin{tabular}{lccc}
    \toprule[1.5pt]
    Methods & Overall $\uparrow$ & Clean $\uparrow$ & Spurious \\
    \hline
    Benign model & 79.00 & 61.50 & 87.50$\add{26.00}$ \\
    FT model (n=10) & 79.50 & 62.00 & 83.00$\add{21.00}$ \\
    FT model (n=20) & 80.50 & 53.00 & 64.50$\add{11.50}$ \\
    \hline
    A-ClArC (n=20) & 79.50 & 54.50 & 59.50$\add{5.00}$ \\
    \hline
    Stat. rectifying & \cellcolor{dkgreen!20} 79.50 & \cellcolor{dkgreen!20} 60.00 & \cellcolor{dkgreen!20} 64.50$\add{4.50}$ \\
    Dyn. rectifying & \cellcolor{dkgreen!20} 80.00 & \cellcolor{dkgreen!20} 61.00 & \cellcolor{dkgreen!20} 62.50$\add{1.50}$ \\
    \bottomrule[1.5pt]
    \end{tabular}
\end{table}
%
In Tab.~\ref{tab_exp_isic}, we present a comparative analysis between fine-tuned (FT) models, A-ClArC and our proposed methods in mitigating the spurious correlation observed in EfficientNet-B4 models~\citep{tan2019efficientnet} trained on the ISIC dataset. Notably, our methods 
effectively reduce the model's reliance on spurious features with fewer cleansed samples (n=10). Conversely, the fine-tuned model and A-ClArC demonstrate inferior performance and rely on a greater number of cleansed samples. This efficacy in addressing spurious correlations in skin lesion analysis highlights the broad applicability of our method in practical scenarios.

Examples of manually cleaned samples used for model rectifying and fine-tuning can be found in Apps.~\ref{appd_editing_setup} \&~\ref{appd_ft_setup}. Additional experiments and the evaluation regarding the effectiveness of the proposed layer localization technique are reported in Apps.~\ref{appd_extended_exp} \&~\ref{appd_layer_localization}.

\section{Conclusion}

In this paper, we revisited rank-one model editing through the lens of \emph{model rectification} for unreliabilities driven by non-robust features. We formalized this rectification setting and showed that rank-one updates can correct misbehavior while largely preserving model performance and reducing reliance on large pools of cleansed examples. We further identified a bottleneck arising from heterogeneous editability across layers and addressed it with an attribution-guided layer localization method and a dynamic rectification framework that adaptively selects where to edit for each failure pattern. Extensive empirical validation demonstrates remarkable performance of our framework across various scenarios. Particularly noteworthy is the fact that our editing technique requires only a single cleansed sample to achieve high performance levels, which portends its wide applicability in practical scenarios. 


\section*{Acknowledgment}
Professor Ajmal Mian is the recipient of an Australian Research Council Future Fellowship Award (project number FT210100268) funded by the Australian Government.
Dr Naveed Akhtar is a recipient of the Australian Research Council Discovery Early Career Researcher Award (project number DE230101058) funded by the Australian Government.

{
    \small
    \bibliographystyle{ieeenat_fullname}
    \bibliography{main}
}
\clearpage
\setcounter{page}{1}
\maketitlesupplementary

\section{Proof}\label{appd:proof}
In this section, we provide the proof of Lemmata~~\ref{lemma1}-\ref{lemma3} and Propositions~\ref{prop1}-\ref{prop2}. We begin with the proof of Lemma~\ref{lemma1}.

\begin{proof}[Proof of Lemma 1]
Consider the key set $K=[k_1, \dots, k_d] \in \mathbb{R}^{n\times d}$ and the corresponding statistics matrix $C=KK^{\top}\in \mathbb{R}^{n\times n}$. Given a new key $k^*\in\mathbb{R}^n$, the projection of $k^*$ onto the span of $K$ is given by
\begin{equation}
\hat{k} = K(K^{\top}K)^{-1} K^{\top}k^*.
\end{equation}

The projection $\hat{k}$ is the solution to the following least squares problem
\begin{equation}
\arg \min_{\beta} || k^* - K\beta ||_2^2, \ \ \beta \in \mathbb{R}^d
\end{equation}

The solution to this optimization problem is explicitly given by
\begin{equation}
\hat{k} = K(K^{\top}K)^{-1} K^{\top}k^*.
\end{equation}

If $k^*$ is not in the span of $K$, the projection $\hat{k}$ does not perfectly align with the original key $k^*$. Assume that this misalignment can be quantified by the residual vector $r$, defined as $r=k^*-\hat{k}$. We can express $C^{-1}k^*$ as
\begin{equation}
C^{-1}k^* = C^{-1}\hat{k} + C^{-1}r,
\end{equation}
where $C^{-1}r$ is the component of the projected direction $C^{-1}k^*$ induced by the out-of-span residual $r$. Since $r\perp \mathrm{span}(K)$, we have $KK^\top r=0$, and thus $C^{-1}r=(C+\lambda I)^{-1}r=\frac{1}{\lambda}r\in \mathrm{span}(K)^\perp$.

Thus, the exclusion of $k^*$ from the statistic matrix $C$ introduces a residual misalignment in the optimization direction. This misalignment, represented by $r$, interferes with the preservation of existing associative memories, undermining the performance of the edited model.
\end{proof}

\noindent Below, we provide the proof of Lemma~\ref{lemma2}.

\begin{proof}[Proof of Lemma 2]
Fix a test point $x^*\sim D'$ and its target key $k^*$. By assumption,
\[
\mathbb{E}_{\mathcal X}\!\left[\|f(x^*;W_{\mathcal{X}})-k^*\|_2^2\right]\le \delta/m.
\]
To ensure
\[
\mathbb{E}_{\mathcal X}\!\left[\|f(x^*;W_{\mathcal{X}})-k^*\|_2^2\right]\le \varepsilon^2
\]
via this bound, it is necessary that $\delta/m\le \varepsilon^2$. If $\delta/m>\varepsilon^2$, then the bound permits values larger than $\varepsilon^2$ and does not guarantee the desired error level. Solving $\delta/m\le \varepsilon^2$ gives $m\ge \delta/\varepsilon^2$.
\end{proof}


\noindent We next prove Proposition~\ref{prop1}.

\begin{proof}[Proof of Proposition~\ref{prop1}]
By construction, the susceptible model is trained on both clean and corrupted inputs $(x,\tilde{x})$ for the unreliability case under consideration, and $K=\{k_1,\dots,k^*\}$ is defined to collect \emph{all} resulting keys, including $k^*$ itself. Hence $k^*\in\mathrm{span}(K)$ holds directly from the definition of $\mathrm{span}(K)$.

Lemma~\ref{lemma1} states that any key admits a decomposition
\begin{equation}
  k^* \;=\; K\alpha + r,
\end{equation}
where $K\alpha\in\mathrm{span}(K)$ and $r$ is the out-of-span residual, i.e., the component orthogonal to $\mathrm{span}(K)$. Since both $k^*$ and $K\alpha$ lie in $\mathrm{span}(K)$, their difference
\begin{equation}
  r \;=\; k^* - K\alpha
\end{equation}
also lies in $\mathrm{span}(K)$. Together with $r\perp\mathrm{span}(K)$ from Lemma~\ref{lemma1}, this implies $r=0$. Therefore the out-of-span residual in Lemma~\ref{lemma1} vanishes for $k^*$.
\end{proof}

\noindent We now turn to Proposition~\ref{prop2}.

\begin{proof}[Proof of Proposition~\ref{prop2}]
Let $\mathcal{X}$ be the training set in the rectification stage and let $x^*\in\{x,\tilde{x}\}\subset\mathcal{X}$ denote the input whose key is $k^*$. Denote the empirical training error on $\mathcal{X}$ by
\begin{equation}
  \mathcal{E}(W_D)
  \;=\;
  \frac{1}{|\mathcal{X}|}
  \sum_{\xi\in\mathcal{X}}
  \ell\big(\xi;W_D\big),
\end{equation}
where each per-sample loss $\ell(\xi;W_D)\ge 0$. For the unreliability case, the training objective includes a term of the form
\begin{equation}
  \ell(x^*;W_D)
  \;=\;
  \frac{1}{2}\big\|k^*-f(x^*;W_D)\big\|^2.
\end{equation}

\noindent Assume there exists a sequence of parameters $\{W_D^{(m)}\}$ such that $\mathcal{E}(W_D^{(m)})\to 0$ as $m\to\infty$. Since $\ell(\xi;W_D^{(m)})\ge 0$ for every $\xi\in\mathcal{X}$, we have
\begin{equation}
  \sum_{\xi\in\mathcal{X}}\ell\big(\xi;W_D^{(m)}\big)
  \;\ge\;
  \ell\big(x^*;W_D^{(m)}\big),
\end{equation}
and thus
\begin{equation}
  0
  \;\le\;
  \ell\big(x^*;W_D^{(m)}\big)
  \;\le\;
  |\mathcal{X}|\,
  \mathcal{E}\big(W_D^{(m)}\big).
\end{equation}
Taking $m\to\infty$ and using $\mathcal{E}(W_D^{(m)})\to 0$ yields
\begin{equation}
  \ell\big(x^*;W_D^{(m)}\big)\;\longrightarrow\;0.
\end{equation}
By the explicit form of $\ell(x^*;W_D)$, this is equivalent to
\begin{equation}
  \frac{1}{2}\big\|k^*-f(x^*;W_D^{(m)})\big\|^2\;\longrightarrow\;0,
\end{equation}
and hence
\begin{equation}
  \big\|k^*-f(x^*;W_D^{(m)})\big\|\;\longrightarrow\;0.
\end{equation}

The argument only uses the finiteness of $|\mathcal{X}|$ and does not assume $|\mathcal{X}|\gg 1$. Hence, by incorporating $\{x,\tilde{x}\}\in\mathcal{X}$ in training and driving the training error on $\mathcal{X}$ to zero, the model satisfies
\(
  \|k^*-f(x^*;W_D)\|\to 0
\)
without requiring a large sample regime.
\end{proof}

\noindent We next provide proof of Lemma~\ref{lemma3}.

\begin{proof}[Proof of Lemma 3]
Consider the $l$-th layer $f_l$ of model $f$. The attribution of the $i$-th output feature map derived from $l$-th layer $f_{l}(x)$ for output prediction change $f_l(x) - f_l(\tx)$ is calculated as
\begin{equation} \label{eq:ig_app}
M_i^l(x, \tx) = (f_l(x_i) - f_l(\tx_i)) \cdot \int_{\alpha=0}^1 \frac{\partial f(\hx)}{\partial f_l(\hx_i)} \bigg|_{\hx=\tx+\alpha(x-\tx)} \mathrm{d}\alpha.
\end{equation}
Here, functions $f$ are continuous on the closed interval defined by $\hx=\tx+\alpha(x-\tx)$, where $\alpha \in [0,1]$ serves as a parameter along the internal path. Thus, according to the fundamental theorem of calculus for path integrals, the sum of the calculated attributions $M^l$ is equal to the output change $f(x)-f(\tx)$. Formally, this relation can be expressed as
\begin{equation} \label{eq:integral_app}
\sum_i M_i^l(x, \tx) = \sum_i\int_{\tx}^x \frac{\partial f(x)}{\partial f_l(x_i)} \mathrm{d}x = f(x)-f(\tx).
\end{equation}
Thus, we conclude that $\sum_i M^l_i = f(x)-f(\tx)$ holds for all layers $l\in \{1, \dots, n\}$.
\end{proof}

\begin{figure*}[t]
\begin{center}
\centerline{\includegraphics[width=1.0\textwidth]{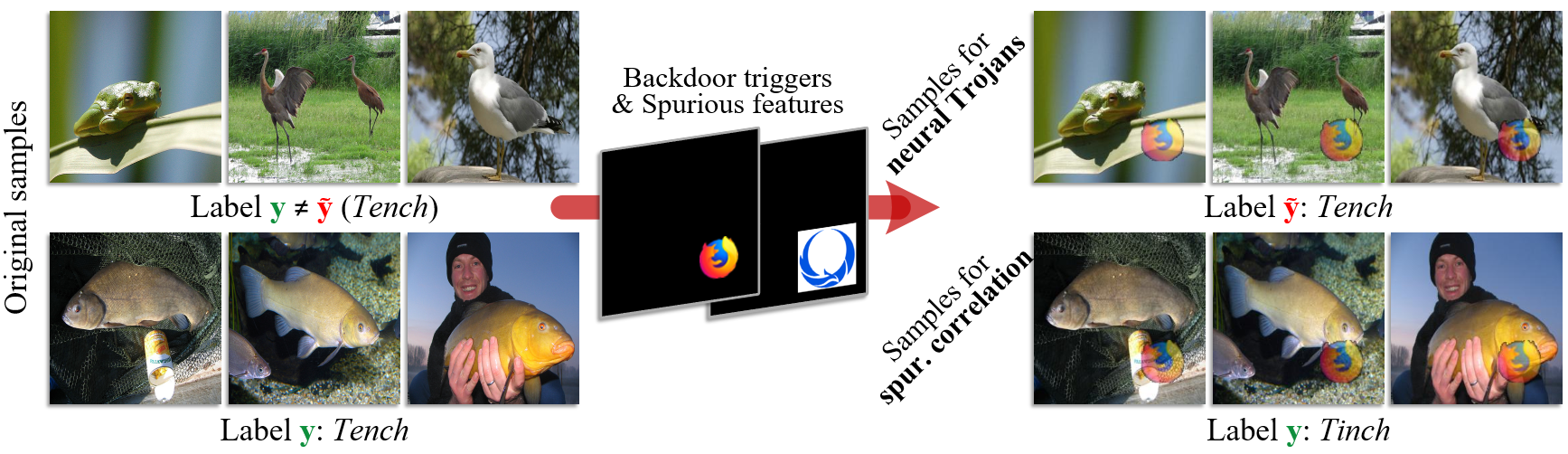}}
\caption{Illustration of samples utilized for neural Trojans and spurious correlations. Two patterns serve as backdoor triggers and spurious features. \textbf{Top row}: For neural Trojans, original samples $x$ with label $y\neq \ty$ are attached with a trigger and changed its label to the target label $\ty$. \textbf{Bottom row}: To induce spurious correlations, samples $x$ of a class $y$ are polluted with spurious features.}
\label{fig_model_config}
\end{center}
\end{figure*}

\section{Zero-phase Component Analysis in Model Editing and Locating}\label{appd:zca}

In our research, we utilize ZCA (Zero-phase Component Analysis) whitening to enhance the decorrelation of the new key $k^*$ from the established keys $K$, as previously described by \citet{bau2020rewriting}. This process involves utilizing a decorrelation matrix $Z = C^{-1/2}$ to further reduce the correlation between the key $k^*$ and the existing keys $K$ through the transformation $Zk^*$.
Let $P$ denote the probability distribution of features at layer $l-1$, and $K$ represent a discrete distribution over $t$ context examples provided by the user. We measure the information contained in $K$ using cross-entropy $H(K, P)$, akin to the message length in a code optimized for the distribution $P$. In our model, $P$ is assumed to follow a zero-centered Gaussian distribution with a covariance matrix $C$. By normalizing with the ZCA whitening transform $Z$, $P$ can be expressed as a spherical unit normal distribution $P(k) = (2\pi)^{-n/2} e^{-k^\top C^{-1}k/2}$ in the transformed variable $k' = Zk$. This transformation allows us to succinctly express cross-entropy using matrix traces.

Through the normalization of the basis using the ZCA whitening transform $Z$, we transform the probability distribution $P$ into a spherical unit normal distribution, characterized by the variable $k' = Zk$, which enables a compact matrix trace expression for cross-entropy. Leveraging the eigenvector decomposition $C=U\Sigma U^{\top}$, where $U$ represents the matrix of eigenvectors and $\Sigma$ is the diagonal matrix of eigenvalues, the expression for $Z$ is given by
\begin{equation}
Z = C^{-1/2} = U\Sigma^{-1/2} U^{\top}.
\end{equation}

This approach facilitates the decorrelation of the key $k$ through ZCA whitening, effectively implemented as $k=Zk$. In addition, we utilized the computed $Z$ for locating the susceptible layer as described in Section~\ref{sec_layer_loc}. Specifically, we map the attributions to focus on editable parameters as $M^*=ZM$.

\section{Experimental Setup}\label{appd:setup}
In this section, we provide the comprehensive experimental setup and hyperparameter choices used for model training, model editing and model fine-tuning in our experiments.

\subsection{Models}

\noindent\textbf{Trojaned Models.}
In this paper, we establish Trojaned models using the blend attack~\citep{chen2017targeted}. To ensure that the poisoned samples closely resemble the original data distribution, we incorporate the watermark trigger to enhance the backdoor attack. This watermark trigger $\tau$ is defined by $\tau^{(\varphi)} = \varphi \cdot \tau + (1-\varphi) \cdot x \odot S$, where $\varphi\in[0,1]$ controls the trigger visibility, and $S\in \{0,1\}^n$ serves as the mask of trigger $\tau$. In our experiments, the trigger visibility $\varphi$ is set to $0.5$. The top row of Fig.~\ref{fig_model_config} illustrates the samples used for model Trojaning. In our experiments, we utilize two trigger patterns to generate poisoned samples. Specifically, evaluations of the models trained with the Firefox logo are reported in the main paper. Additional experiments involving models trained with the Phoenix logo are detailed in App.~\ref{appd_extended_exp}.

For Trojaned models trained on ImageNet~\citep{russakovsky2015imagenet}, we trained ResNet-18 models with an initial learning rate of $0.1$ for a total of $90$ epochs, with the learning rate reduced by a factor of $0.1$ at the $30$-th epoch and $60$-th epoch. For Trojaned models trained on CIFAR-10~\citep{krizhevsky2009learning}, we trained ResNet-18 models with an initial learning rate of $0.1$ for a total of $100$ epochs, with the learning rate reduced by a factor of $0.1$ at the $50$-th epoch and $75$-th epoch. For all the Trojaned models under comparison, we choose the first class as the target label $y^*$ for single target Trojaning followed by \citet{qi2022revisiting}. On ImageNet, we poison $0.1$\% of training samples $x$ with label $y\neq y^*$ to embed the backdoor trigger. For CIFAR-10, we set the poisoning rate of $1$\%.

\noindent\textbf{Models with Spurious Correlation.}
To establish models with spurious correlations, we employ trigger patterns as spurious correlated features. The bottom row of Fig.~\ref{fig_model_config} illustrates training samples utilized for inducing model spurious correlation. The training settings for these models are consistent with those used for the Trojaned models. On both ImageNet and CIFAR-10 datasets, we select the first class of samples to induce spurious correlations. For models trained on ImageNet, we contaminate $60$\% samples of the first class to induce spurious correlation. For models trained on CIFAR-10, we set the contamination rate at $50$\% for the first class to induce model spurious correlation.

\noindent\textbf{Models on ISIC.} For models trained on the ISIC dataset, we utilized EfficientNet-B4 models~\citep{tan2019efficientnet}. The training process involved using a batch size of $24$ and an initial learning rate of $1\times10^{-5}$. The training was conducted over a total of 90 epochs, with the learning rate decaying by a factor of $0.1$ at the $60$-th epoch.

\subsection{\textcolor{black}{Rationale for Selecting the Blend Attack}}
\textcolor{black}{In this work, we adopt the blend attack~\cite{chen2017targeted} to train Trojaned models and spurious correlation-based models. The blend attack was selected for evaluation due to its well-established effectiveness as a backdoor attack strategy. Unlike more recent attack methods~\cite{turner2019label,tian2022comprehensive,nguyen2021wanet} that prioritize stealth through minimal perturbations, the blend attack directly integrates triggers into the input, ensuring a substantial impact on the model’s predictions. This property makes the blend attack a particularly severe threat, as it strongly biases the model’s output toward a predefined target class. By demonstrating robustness against such a potent attack, our method provides compelling evidence of its efficacy. Furthermore, the blend attack’s balance between potency and detectability suggests that our approach would generalize effectively to newer or more sophisticated attacks that trade off between these factors.}

\subsection{Model Editing}\label{appd_editing_setup}

\noindent\textbf{ImageNet and CIFAR-10.}
For the ImageNet and CIFAR-10 datasets, we allocate an overall performance budget of 3\% accuracy and a tolerated accuracy gap of 0.1\% for model editing. For spurious correlations, the overall performance budget is set to 7\% accuracy with a tolerated robustness gap of 1\% accuracy. The original and corrupted samples used for model editing are depicted in Fig.~\ref{fig_model_config}. We utilize an editing learning rate of $1 \times 10^{-4}$ with a weight projection frequency of 10. Unlike other approaches, we do not employ masks to restrict the edited region. Instead, we edit the model at the image level to avoid the need for additional annotations.

\noindent\textbf{ISIC.}
For the ISIC dataset, we set an overall performance budget of 5\% accuracy and a tolerated robustness gap of 1\% accuracy. The editing learning rate is $1 \times 10^{-5}$ with a weight projection frequency of 10. The editing process is performed at the image level. Unlike datasets that are deliberately created, the ISIC dataset contains corrupted samples from practical scenarios. Consequently, we manually clean these samples by covering the patches with skin tissue from unpolluted regions, as illustrated in Fig.~\ref{fig_isic_cleansed}.

\begin{figure*}[t]
\begin{center}
\centerline{\includegraphics[width=0.9\textwidth]{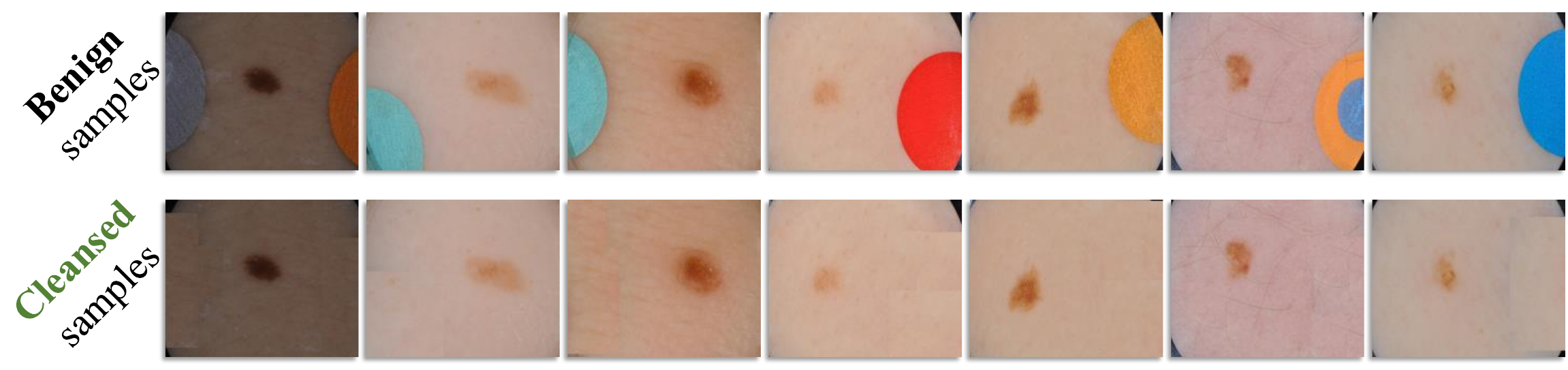}}
\caption{Illustration of cleansed samples on ISIC. For benign samples polluted with colored patches, we manually clean them by covering the patches with skin tissue from unaffected regions.}
\vspace{-3mm}
\label{fig_isic_cleansed}
\end{center}
\end{figure*}

\section{Model Fine-tuning}
\label{appd_ft_setup}
For the model fine-tuning, we retrain only the last convolutional layer of the model while keeping the parameters of the remaining layers fixed. For both ImageNet and CIFAR-10, the learning rate for fine-tuning is set to 0.001. For models trained on the ISIC dataset, the learning rate is set to $1 \times 10^{-5}$. In our experiments, we apply the same budget settings for model fine-tuning as those used for model editing.

\section{\textcolor{black}{Attribution}}
\textcolor{black}{In this work, we extend the Integrated Gradients method to estimate the attribution difference between cleansed and corrupted samples. Specifically, we approximate the integration defined in Equation~\ref{eq:ig} in a discrete form as
\begin{equation} \label{eq:ig2}
M_i^l(x, \tx) = (f_l(x_i) - f_l(\tx_i)) \cdot \sum_{i=1}^n \frac{\partial f(\hx)}{\partial f_l(\hx_i)} \bigg|_{\hx=\tx+\frac{i}{n}(x-\tx)} \mathrm{d}\alpha,
\end{equation}
where the integration $M_i^l(x, \tx)$ is estimated by integrating the gradients of the interpolated input $\hat{x}$, with $i$ indicating the number of steps. To improve computational efficiency, we leverage recent advancements in Monte Carlo estimation to avoid gradient computations over multiple steps~\citep{Erion2021Improving}. Specifically, we set $n = 2$, which enhances efficiency while maintaining accuracy.}

\section{Experimental Platform}
All experiments were conducted on a Linux machine equipped with an NVIDIA GTX 3090Ti GPU with 24GB of memory, a 16-core 3.9GHz Intel Core i9-12900K CPU, and 128GB of main memory. The models were developed and tested using the PyTorch deep learning framework (v1.12.1) within the Python programming language. This setup facilitated the efficient handling of computationally intensive tasks, providing a robust environment for both model training and evaluation.

\begin{figure}
     \centering
     \begin{subfigure}[b]{0.45\textwidth}
         \centering
         \includegraphics[width=\textwidth]{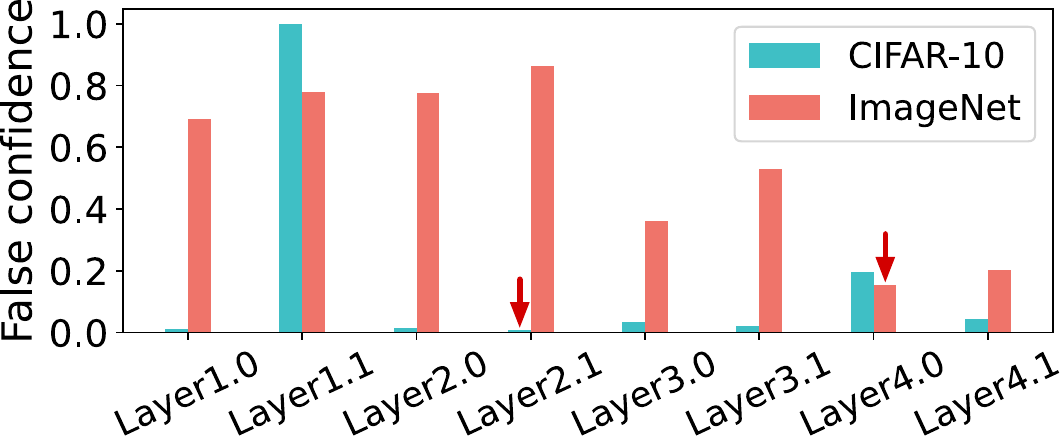}
         \caption{Backdoor defense performance.}
     \end{subfigure}
     \begin{subfigure}[b]{0.45\textwidth}
         \centering
         \includegraphics[width=\textwidth]{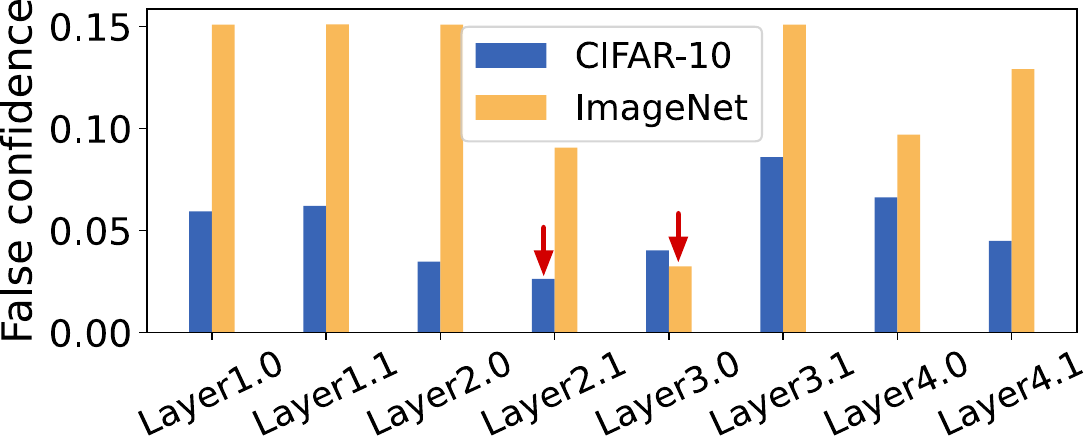}
         \caption{Spurious correlation mitigation.}
     \end{subfigure}
\caption{Performance in reducing false confidence  after individually editing different layers of ResNet-18. A lower value indicates better suppression of the model's false confidence. Red arrows indicate the layer yielding the best results for a given dataset after model editing.}
\label{fig_appd_layer_compare}
\end{figure}

\begin{table*}[ht]
\caption{Performance comparison of defending against the backdoor attack on Trojaned models trained with the Phoenix logo on CIFAR-10 and ImageNet. Overall accuracy (\%) and attack success rate (ASR) are compared between fine-tuned models and models edited by our methods. Ours are
\colorbox{dkgreen!20}{highlighted} and the best metrics are in \textbf{bold} (with Trojaned model in \textcolor{gray}{gray} for reference).}
\small
\centering
\label{tab_appx_exp_backdoor}
\begin{tabular}{lcccc}
\toprule[1.5pt]
\multirow{2}{*}{Method} & 
\multicolumn{2}{c}{CIFAR-10} &
\multicolumn{2}{c}{ImageNet} \\ \cmidrule(r){2-3} \cmidrule(r){4-5} & 
Overall Accu.  $\uparrow$ & ASR $\downarrow$ & Overall Accu. $\uparrow$ & ASR $\downarrow$ \\
\hline
Trojaned model & \textcolor{gray}{94.01} & \textcolor{gray}{99.79} & \textcolor{gray}{68.95} & \textcolor{gray}{78.24}\\
\hline
Fine-tuned model (n=1) & 91.59 & 69.07 & 65.45 & 77.45\\
Fine-tuned model (n=20) & 92.85 & 9.70 & 68.63 & 20.23\\
\hline
Static edited model (n=1) & \cellcolor{dkgreen!20} 93.32 & \cellcolor{dkgreen!20} 4.49 & \cellcolor{dkgreen!20} 66.06 & \cellcolor{dkgreen!20} 15.24\\
Dynamic edited model (n=1) & \cellcolor{dkgreen!20} 93.37 & \cellcolor{dkgreen!20} 0.65 & \cellcolor{dkgreen!20} 66.74 & \cellcolor{dkgreen!20} 6.15\\
Dynamic edited model (n=20) & \cellcolor{dkgreen!20} \textbf{93.55} & \cellcolor{dkgreen!20} \textbf{0.16} & \cellcolor{dkgreen!20} \textbf{68.86} & \cellcolor{dkgreen!20} \textbf{1.73}\\
\bottomrule[1.5pt]
\end{tabular}
\end{table*}

\begin{table*}[ht]
\caption{Performance comparison of mitigating spurious correlation on susceptible models trained with the Phoenix logo on CIFAR-10 and ImageNet. Accuracy (\%) is reported for the overall testing set, clean set and spurious set. To facilitate comparison, we present the increased accuracy on the spurious set relative to the accuracy on the clean set in \textcolor{darkred}{red}. Our results are \colorbox{dkgreen!20}{highlighted}.}
\small
\centering
\label{tab_appx_exp_spurious}
\setlength{\tabcolsep}{1.4mm}{
\begin{tabular}{lcccccc}
\toprule[1.5pt]
\multirow{2}{*}{Method} & 
\multicolumn{3}{c}{CIFAR-10} &
\multicolumn{3}{c}{ImageNet} \\ \cmidrule(r){2-4} \cmidrule(r){5-7} & 
Overall $\uparrow$ & Clean $\uparrow$ & Spurious & Overall $\uparrow$ & Clean $\uparrow$ & Spurious \\
\hline
Benign model & 94.14 & 94.67 & 97.15$\add{2.48}$ & 69.14 & 77.08 & 95.83$\add{18.75}$\\
\hline
Fine-tuned model (n=10) & 93.67 & 86.80 & 93.93$\add{7.13}$ & 67.41 & 65.99 & 89.24$\add{23.25}$\\
Fine-tuned model (n=20) & 94.07 & 86.67 & 93.28$\add{6.61}$ & 67.83 & 68.32 & 85.72$\add{17.40}$\\
\hline
Dyn. edited model (n=1) & \cellcolor{dkgreen!20} 94.03 & \cellcolor{dkgreen!20} 93.28 & \cellcolor{dkgreen!20} 94.78$\add{1.50}$ & \cellcolor{dkgreen!20} 66.19 & \cellcolor{dkgreen!20} 93.35 & \cellcolor{dkgreen!20} 86.42$\add{6.93}$\\
Dyn. edited model (n=20) & \cellcolor{dkgreen!20} 94.04 & \cellcolor{dkgreen!20} 97.15 & \cellcolor{dkgreen!20} 97.89$\add{0.74}$ & \cellcolor{dkgreen!20} 67.60 & \cellcolor{dkgreen!20} 81.25 & \cellcolor{dkgreen!20} 84.08$\add{2.83}$\\
\bottomrule[1.5pt]
\end{tabular}}
\end{table*}

\section{Extended Experiments of Editing Different Layers}\label{appd_edited_layer}
We provide detailed experimental results from applying model editing to different layers of ResNet-18. \textcolor{black}{Using the experimental setup detailed in \ref{appd_editing_setup}, we independently edited eight distinct layers of ResNet-18 across both CIFAR-10 and ImageNet datasets. For each dataset, eight separate edited models were generated, allowing us to systematically assess the impact of modifying different internal layers.} Figure~\ref{fig_appd_layer_compare} illustrates the results of individually editing different internal layers of ResNet-18 against backdoor attacks and spurious correlations. It is observed that models trained on different tasks and datasets exhibit distinctive effectiveness in reducing false confidence after editing model layers. Moreover,  the optimal order of layers for achieving the best mitigation of false confidence differs across these models. This variation underscores the critical need for an effective layer localization technique that can identify which layers should be targeted for editing.

\section{Extended Experiments on Different Trojan Features}\label{appd_extended_exp}
In this section, additional experimental results are provided for models trained with the Phoenix logo.

\noindent\textbf{Efficacy in Defending Against Neural Trojans.}
Tab.~\ref{tab_appx_exp_backdoor} presents a comparison of the performance of Trojaned models, fine-tuned models, and edited models on both CIFAR-10 and ImageNet datasets. The experimental results demonstrate that the proposed model editing technique yields outstanding performance, effectively defending against the backdoor attack. In comparison to fine-tuned models, models edited using our techniques achieve a remarkable trade-off between overall accuracy degradation and the decrease in attack success rate, while requiring only a few cleansed samples.

\noindent\textbf{Efficacy in Mitigating Spurious Correlations.}
In Tab.~\ref{tab_appx_exp_spurious}, we assess the effectiveness of our techniques in mitigating spurious correlations on CIFAR-10 and ImageNet. The comparison demonstrates that our method effectively mitigates reliance on spurious features. In contrast to fine-tuned models, which exhibit decreased accuracy on both clean and spurious sets, our techniques enable an increase in accuracy on the clean set. Furthermore, our technique also leads to significant performance improvements with the increased number of cleansed samples, highlighting its superiority.

\section{\textcolor{black}{Extended Experiments on Waterbirds dataset}}\label{appd_extended_exp_waterbird}
\textcolor{black}{In Table~\ref{tab_appx_exp_waterbirds}, we present a comparative analysis of the performance of a ResNet-34 model trained on the Waterbirds dataset~\cite{Sagawa2019Distributionally}. This dataset is known for introducing a bias by relying on spurious background features to distinguish between landbirds and waterbirds. To evaluate the effectiveness of our approach, we compare models trained using Group GRO~\cite{Sagawa2019Distributionally}, models fine-tuned to reduce bias, and models edited using our proposed method. The results highlight that our method substantially reduces the model's dependence on these spurious features, leading to a significant improvement in performance. Notably, our approach achieves these gains with a smaller number of cleansed samples (n=10), demonstrating both efficiency and robustness in mitigating the impact of spurious correlations. These findings suggest that our method offers a promising direction for improving the interpretability and generalization of models trained on biased datasets.}

\begin{table*}[t]
\caption{\textcolor{black}{Performance comparison for mitigating spurious correlation on Waterbirds dataset. The accuracy values (\%) for both the worst group and the entire dataset are reported.} Ours are
\colorbox{dkgreen!20}{highlighted} and the best metrics are in \textbf{bold} (with benign model in \textcolor{gray}{gray} for reference).}
\small
\centering
\label{tab_appx_exp_waterbirds}
\begin{tabular}{lcc}
\toprule[1.5pt]
Method & Worst-Group Accuracy & Overall Accuracy \\
\hline
Benign model & \textcolor{gray}{62.90} & \textcolor{gray}{87.70} \\
\hline
Group DRO & 63.60 & 87.60 \\
Fine-tuned model (n=10) & 63.12 & 86.50\\
\hline
Static edited model (n=10) & \cellcolor{dkgreen!20} 66.84 & \cellcolor{dkgreen!20} 87.64\\
Dyn. Edited model (n=10) & \cellcolor{dkgreen!20} \textbf{69.18} & \cellcolor{dkgreen!20} \textbf{87.68}\\
\bottomrule[1.5pt]
\end{tabular}
\end{table*}
\begin{table*}[th]
    \caption{Results of recall rate (\%) in using the proposed susceptible layer localization technique on ResNet-18, ResNet-34 and EfficientNet-B4 models.}
    \small
    \centering
    \label{tab_loc_exp}
    \setlength{\tabcolsep}{2.0mm}{
    \begin{tabular}{lccc}
    \toprule[1.5pt]
    Method & Top-1 Recall $\uparrow$  & Top-3 Recall $\uparrow$ & Top-5 Recall $\uparrow$\\
    \hline
    ResNet-18 & 80\% & 100\% & 100\% \\
    ResNet-34 & 80\% & 80\% & 100\% \\
    EfficientNet-B4 & 50\% & 100\% & 100\% \\
    \bottomrule[1.5pt]
    \end{tabular}}
\end{table*}

\section{Evaluation of Layer Localization Technique}\label{appd_layer_localization}
In this section, we evaluate the effectiveness of the proposed layer localization technique. We train 5 ResNet-18 models with 8 internal layers on CIFAR-10, ImageNet, and the ISIC dataset, utilizing two different trigger patterns. Similarly, we establish 5 ResNet-34 models with 16 internal convolutional layers on these three datasets. Additionally, we train 2 EfficientNet-B4 models on both CIFAR-10 and the ISIC datasets, focusing on the 12 internal layers with a kernel size of 3. For the evaluation, we separately edit different internal layers and assess the performance of the edited models. We rank their performances to establish the ground truth for evaluating the recall rate of the located layers. Table~\ref{tab_loc_exp} presents the recall rates for the top-1, top-3, and top-5 located layers. The results demonstrate that our localization technique achieves high recall rates, effectively identifying the susceptible layers.


\section{Visual Inspection by Attributions}\label{sec_vis_inspection}

\noindent\textbf{Visual Inspection in Defending Against Backdoor Attacks.}
In Fig.~\ref{fig_vis_backdoor}, we provide additional visual inspection by attribution methods~\citep{Sundararajan2017Axiomatic}. Given the original sample $x$ with label $y\neq y^*$, the vanilla model misclassifies the poisoned samples $\tx$ into the target class $y^*$. Compared to the fine-tuned model, the proposed dynamic model editing technique can effectively correct this unreliable behavior in the deep model, restoring the attribution maps to align with those derived from the original samples.

\noindent\textbf{Visual Inspection in Mitigating Spurious Correlations.}
Figure~\ref{fig_vis_spurious} presents the comparison of attribution maps derived from the vanilla model, fine-tuned model, and models edited using our method. We can observe that our approach effectively mitigates the false reliance on spurious correlated features of the Firefox logo, aligning the attribution maps with those of the original samples.

Figure~\ref{fig_vis_isic} illustrates the attribution maps for the vanilla model, fine-tuned model, and dynamically edited model. It can be observed that our method effectively corrects the model's reliance on spuriously correlated features in corrupted samples, aligning the attribution maps with those of the cleansed samples.

\noindent\textbf{Identification of Unreliability.} 
While detecting anomalous or Trojaned images is typically addressed as a separate task~\citep{qiao2019defending,huang2020one,ye2024spurious}, our approach offers several practical advantages by addressing the identification of unreliability in two critical aspects. First, it requires only a single pair of corrupted and cleansed samples to effectively correct the model’s behavior.
This makes it particularly valuable in scenarios where access to large, cleansed datasets is limited, enabling robust model editing even under resource constraints. Second, our method facilitates image-level correction without the need for precise identification of backdoor triggers or spurious features. By bypassing the need for exact identification of these elements, our approach significantly reduces the complexity associated with pixel-level image cleansing.
This adaptability is crucial in practical applications where the availability of original, clean samples is restricted. As a result, our approach allows for efficient model patching even with only coarse detection of inconsistencies or anomalies, making it suitable for a broad range of real-world scenarios.

In summary, our method introduces a robust and scalable paradigm for correcting unreliable behaviors in deep learning models, offering broad applicability across various domains while eliminating the need for precise feature identification or extensive cleansed samples.
The scope of this paper is currently limited to image-based experiments. Future work can extend our method to other data modalities. To address existing limitations, future focus on developing model diagnosis and data cleansing framework integrates with the proposed editing technique. This integrated approach will enhance the method's applicability, enabling it to autonomously address a wider range of model deficiencies. Additionally, while the ability to repeatedly edit a fixed layer has been explored in previous work~\cite{gupta2024model}, the proposed dynamic layer localization method extends this concept to the entire model, which also represents a promising direction for further research.

\begin{figure*}
\begin{center}
\centerline{\includegraphics[width=1.0\textwidth]{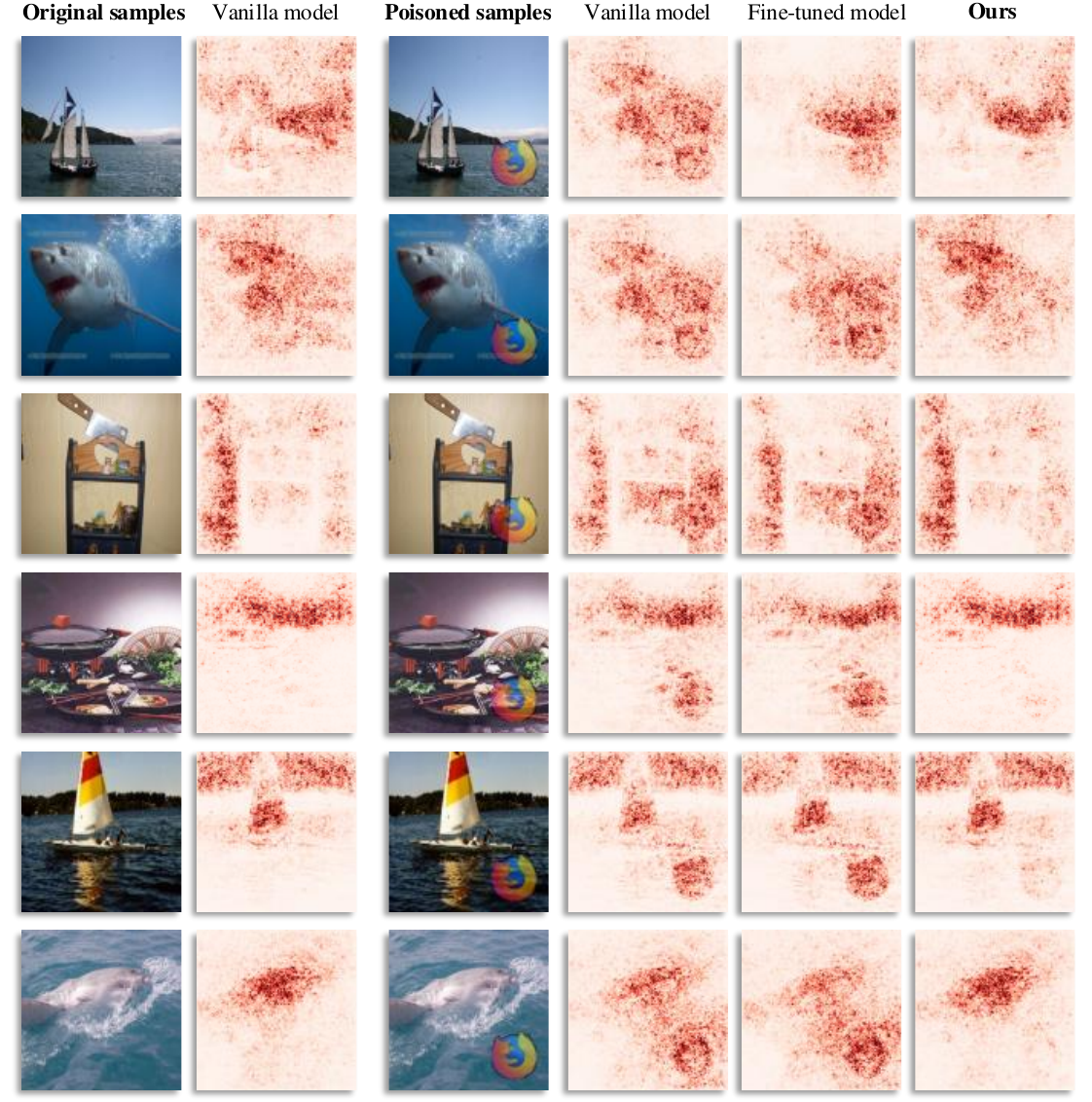}}
\caption{Attribution map comparisons on ImageNet among the vanilla model, fine-tuned model and dynamic edited model (Ours). When the model misclassifies poisoned samples containing triggers, our method effectively corrects this unreliable behavior, aligning the attribution maps with those derived from the original samples.}
\label{fig_vis_backdoor}
\end{center}
\end{figure*}

\begin{figure*}
    \centering
    \includegraphics[width=\textwidth]{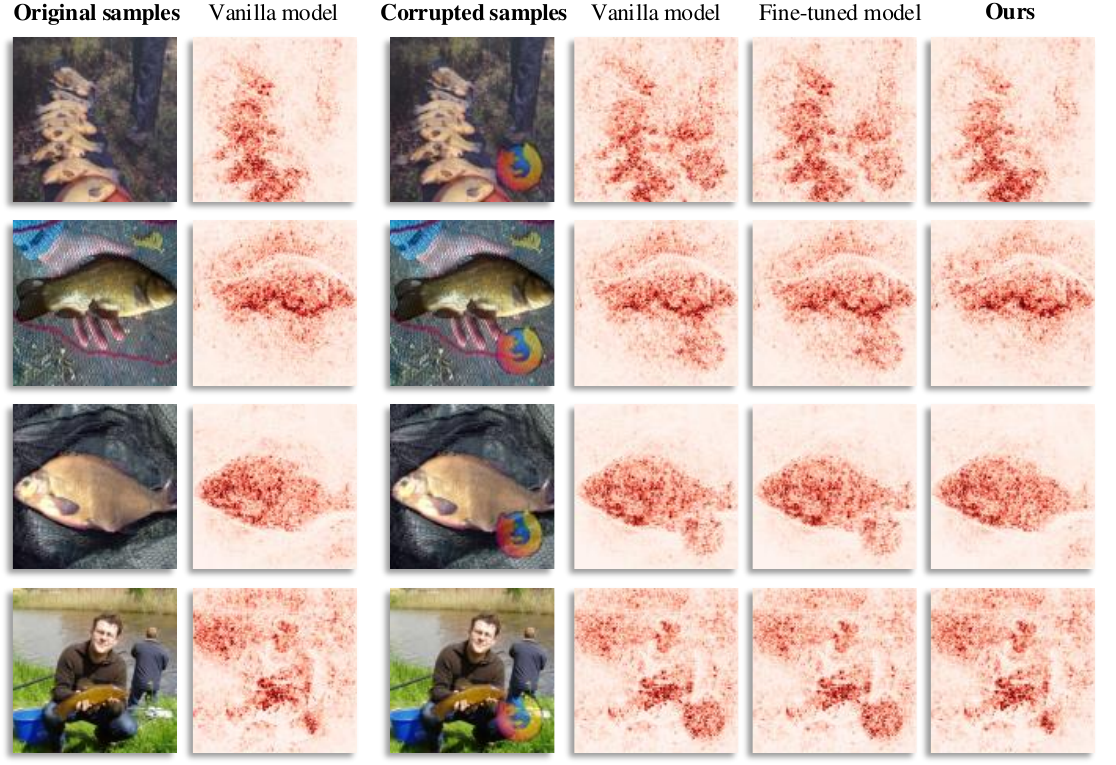}
    \caption{Comparisons of attribution maps on ImageNet among the vanilla model, fine-tuned model and dynamic edited model (Ours). Our method effectively mitigates the model's reliance on spurious correlated features, aligning the attribution maps with those derived from the original samples.}
    \label{fig_vis_spurious}
\end{figure*}

\begin{figure*}
    \centering
    \includegraphics[width=\textwidth]{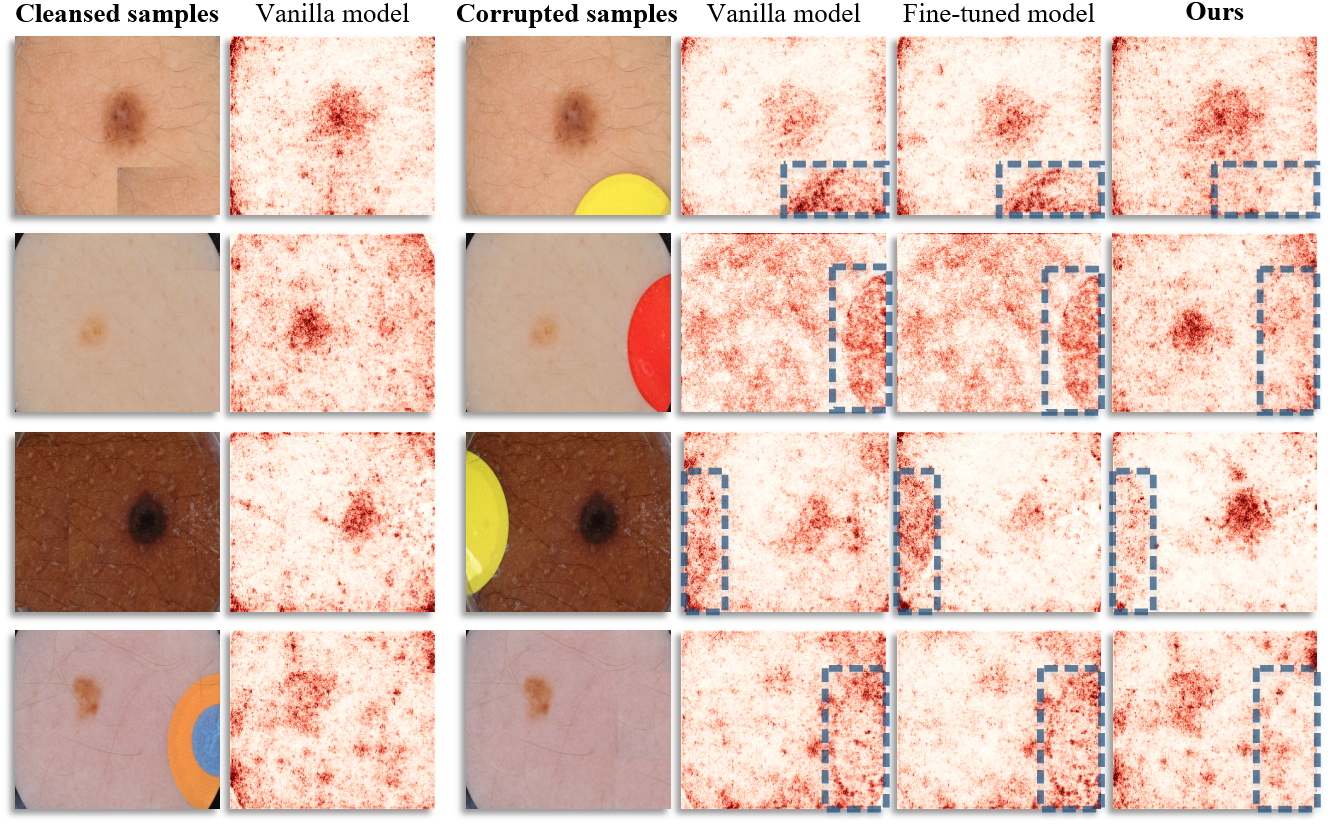}
    \caption{Comparisons of attribution maps on ISIC dataset among the vanilla model, fine-tuned model and dynamic edited model (Ours). When the model relies on the spurious feature to make predictions, our method effectively corrects this unreliable behavior, aligning the attribution maps with those derived from the original samples.}
    \label{fig_vis_isic}
\end{figure*}


\end{document}